\documentclass{article}

\usepackage[final]{nips_2016}

\usepackage[utf8]{inputenc}
\usepackage[T1]{fontenc}
\usepackage{hyperref}
\usepackage{url}
\usepackage{booktabs}
\usepackage{amsfonts}
\usepackage{nicefrac}
\usepackage{microtype}

\usepackage{multibib}
\newcites{app}{Appendix References}
\usepackage{algorithm}
\usepackage{algorithmic}
\usepackage{graphicx}
\usepackage{subfigure} 
\usepackage{latexsym}
\usepackage{amsmath}
\usepackage{amsthm}
\usepackage{amssymb}
\usepackage[mathscr]{euscript}
\usepackage[whileod]{algochl}

\hypersetup{
  colorlinks   = true,
  urlcolor     = blue,
  linkcolor    = blue,
  citecolor   = blue
}

\makeatletter
\newtheorem*{rep@theorem}{\rep@title}
\newcommand{\newreptheorem}[2]{%
\newenvironment{rep#1}[1]{%
\def\rep@title{#2 \ref{##1}}%
\begin{rep@theorem}}%
{\end{rep@theorem}}}
\makeatother

\def\Nset{\mathbb{N}}
\def\Rset{\mathbb{R}}

\DeclareMathOperator*{\E}{\mathbb E}
\DeclareMathOperator*{\argmax}{argmax}
\DeclareMathOperator*{\argmin}{argmin}
\DeclareMathOperator{\orig}{orig}
\DeclareMathOperator{\dest}{dest}
\DeclareMathOperator{\weight}{weight}
\DeclareMathOperator*{\Det}{Det}
\DeclareMathOperator{\conv}{conv}

\newcommand{\h}{\widehat}
\newcommand{\tl}{\widetilde}
\newcommand{\ov}{\overline}

\newcommand{\cD}{{\mathcal D}}

\newcommand{\cF}{{\mathcal F}}
\newcommand{\cH}{{\mathcal H}}

\newcommand{\cX}{{\mathcal X}}
\newcommand{\cY}{{\mathcal Y}}

\newcommand{\cU}{{\mathcal U}}

\newcommand{\I}{\mathbf 1}
\newcommand{\mat}[1]{{\mathbf #1}}

\newcommand{\bh}{\mat{h}}
\newcommand{\bu}{\mat{u}}
\newcommand{\bw}{\mat{w}}
\newcommand{\bx}{\mat{x}}

\newcommand{\bz}{\mat{z}}

\newcommand{\bPsi}{\mat{\Psi}}
\newcommand{\Alpha}{{\boldsymbol \alpha}}
\newcommand{\bpsi}{{\boldsymbol \psi}}
\newcommand{\bsigma}{{\boldsymbol \sigma}}
\newcommand{\be}{{\boldsymbol \epsilon}}
\newcommand{\bup}{{\boldsymbol \upsilon}}

\renewcommand{\sf}{\mathsf}
\newcommand{\hh}{{\sf h}}

\newcommand{\qq}{{\sf q}}
\newcommand{\QQ}{{\sf Q}}

\newcommand{\scrA}{{\mathscr A}}
\newcommand{\scrB}{{\mathscr B}}
\newcommand{\scrC}{{\mathscr C}}

\newcommand{\scrM}{{\mathscr M}}
\newcommand{\scrN}{{\mathscr N}}
\newcommand{\scrT}{{\mathscr T}}
\newcommand{\scrU}{{\mathscr U}}

\newcommand{\scrY}{{\mathscr Y}}

\newcommand{\Rad}{\mathfrak R}

\newcommand{\loss}{{\mathsf L}}
\newcommand{\VCRF}{VCRF}
\newcommand{\StructBoost}{StructBoost}
\newcommand{\VStructBoost}{VStructBoost}
\newcommand{\n}{\mat{n}}
\newcommand{\N}{\mat{N}}

\newcommand{\F}{\Phi^*}

\newcommand{\e}{\epsilon}
\newcommand{\ve}{\varepsilon}
\renewcommand{\d}{\delta}

\newcommand{\set}[2][]{#1 \{ #2 #1 \} }
\newcommand{\ignore}[1]{}

\newtheorem{theorem}{Theorem}
\newtheorem{lemma}[theorem]{Lemma}

\newtheorem{corollary}[theorem]{Corollary}

\newreptheorem{theorem}{Theorem}
\newreptheorem{lemma}{Lemma}

\title{Structured Prediction Theory Based on\\ Factor Graph Complexity}

\author{
Corinna Cortes \\
\phantom{xXX} Google Research \phantom{XXXx}\\
New York, NY 10011 \\
\texttt{\small corinna@google.com} \\
\And
Vitaly Kuznetsov \\
Google Research\\
New York, NY 10011 \\
\texttt{\small vitaly@cims.nyu.edu} \\
\AND
Mehryar Mohri \\
Courant Institute and Google\\
New York, NY 10012 \\
\texttt{\small mohri@cims.nyu.edu} \\
\And
Scott Yang \\
Courant Institute\\
\ignore{New York University\\}
New York, NY 10012 \\
\texttt{\small yangs@cims.nyu.edu}
}

\begin{document}

\maketitle

\begin{abstract}
  We present a general theoretical analysis of structured prediction
  with a series of new results. We give new data-dependent margin
  guarantees for structured prediction for a very wide family of loss
  functions and a general family of hypotheses, with an arbitrary
  factor graph decomposition. These are the tightest margin bounds
  known for both standard multi-class and general structured
  prediction problems.  Our guarantees are expressed in terms of a
  data-dependent complexity measure, \emph{factor graph complexity},
  which we show can be estimated from data and bounded in terms of
  familiar quantities for several commonly used hypothesis sets along
  with a sparsity measure for features and graphs. Our proof
  techniques include generalizations of Talagrand's contraction lemma
  that can be of independent interest.

  We further extend our theory by leveraging the principle of Voted
  Risk Minimization (VRM) and show that learning is possible even with
  complex factor graphs.  We present new learning bounds for this
  advanced setting, which we use to design two new algorithms,
  \emph{Voted Conditional Random Field} (VCRF) and \emph{Voted
    Structured Boosting} (StructBoost).  These algorithms can make use
  of complex features and factor graphs and yet benefit from
  favorable learning guarantees. We also report the results of
  experiments with VCRF on several datasets to validate our theory.

\end{abstract}

\section{Introduction}
\label{sec:intro}

Structured prediction covers a broad family of important learning
problems. These include key tasks in natural language processing such
as part-of-speech tagging, parsing, machine translation, and
named-entity recognition, important areas in computer
vision such as image segmentation and object recognition, and also
crucial areas in speech processing such as pronunciation modeling and
speech recognition.

In all these problems, the output space admits some structure. This
may be a sequence of tags as in part-of-speech tagging, a parse tree
as in context-free parsing, an acyclic graph as in dependency parsing,
or labels of image segments as in object detection. Another property
common to these tasks is that, in each case, the natural loss function
admits a decomposition along the output substructures.  As an example,
the loss function may be the Hamming loss as in part-of-speech
tagging, or it may be the edit-distance, which is widely used in
natural language and speech processing.

The output structure and corresponding loss function make these
problems significantly different from the (unstructured) binary
classification problems extensively studied in learning theory. In
recent years, a number of different algorithms have been designed for
structured prediction, including Conditional Random Field (CRF)
\citep{LaffertyMcCallumPereira2001}, StructSVM
\citep{TsochantaridisJoachimsHofmannAltun2005}, Maximum-Margin Markov
Network (M3N) \citep{TaskarGuestrinKoller2003}, a kernel-regression
algorithm \citep{CortesMohriWeston2007}, and search-based approaches
such as
\citep{HalDaumeIIILangfordMarcu2009,DoppaFernTadepalli2014,LamDoppaTodorovicDietterich2015,ChangKrishnamurthyAgarwalHalDaumeIIILangford15,RossGordonBagnell2011}.
More recently, deep learning techniques have also been developed for
tasks including part-of-speech tagging
\citep{JurafskyMartin2009, VinyalsKaiserKooPetrovSutskeverHinton2015},
named-entity recognition \citep{NadeauSekine2007}, machine translation
\citep{ZhangSunLi2008}, image segmentation \citep{LucchiLiFua2013},
and image annotation \citep{VinyalsToshevBengioErhan2015}.

However, in contrast to the plethora of algorithms, there have been
relatively few studies devoted to the theoretical understanding of
structured prediction
\citep{BakirHofmannScholkopfSmolaTaskarVishwanathan2007}.  Existing
learning guarantees hold primarily for simple losses such as
the Hamming loss
\citep{TaskarGuestrinKoller2003,CortesKuznetsovMohri2014,Collins2001}
and do not cover other natural losses such as the edit-distance. They
also typically only apply to specific factor graph models.  The main
exception is the work of \cite{McAllester2007}, which provides
PAC-Bayesian guarantees for arbitrary losses, though only in the
special case of randomized algorithms using linear (count-based)
hypotheses.  \ignore{ Two exceptions are the work of
  \cite{McAllester2007}, which provides PAC-Bayesian guarantees for
  arbitrary losses in the special case of randomized algorithms with
  linear (count-based) hypotheses, and the recent work of
  \cite{LondonHuangGetoor2016}, which gives PAC-Bayesian guarantees
  for a specific family of stable losses.}

This paper presents a general theoretical analysis of structured
prediction with a series of new results.  We give new data-dependent
margin guarantees for structured prediction for a broad family of loss
functions and a general family of hypotheses, with an arbitrary factor
graph decomposition.  These are the tightest margin bounds known for
both standard multi-class and general structured prediction
problems. For special cases studied in the past, our learning bounds
match or improve upon the previously best bounds (see
Section~\ref{sec:special}). In particular, our bounds improve upon
those of \cite{TaskarGuestrinKoller2003}. Our guarantees are expressed
in terms of a data-dependent complexity measure, \emph{factor graph
  complexity}, which we show can be estimated from data and bounded in
terms of familiar quantities for several commonly used hypothesis
sets along with a sparsity measure for features and graphs.

We further extend our theory by leveraging the principle of Voted Risk
Minimization (VRM) and show that learning is possible even with
complex factor graphs.  We present new learning bounds for this
advanced setting, which we use to design two new 
algorithms, \emph{Voted Conditional Random Field} (VCRF) and
\emph{Voted Structured Boosting} (StructBoost).  These algorithms can
make use of complex features and factor graphs and yet benefit
from favorable learning guarantees.  As a proof of concept validating
our theory, we also report the results of experiments with VCRF on
several datasets.

The paper is organized as follows. In Section~\ref{sec:scenario} we
introduce the notation and definitions relevant to our discussion of
structured prediction. In Section~\ref{sec:bounds}, we derive a series
of new learning guarantees for structured prediction, which are then
used to prove the VRM principle in
Section~\ref{sec:vrm}. Section~\ref{sec:algo} develops the algorithmic
framework which is directly based on our theory.  In
Section~\ref{sec:experiments}, we provide some preliminary
experimental results that serve as a proof of concept for our theory.

\section{Preliminaries}
\label{sec:scenario}

Let $\cX$ denote the input space and $\cY$ the output space. In
structured prediction, the output space may be a set of sequences,
images, graphs, parse trees, lists, or some other (typically discrete)
objects admitting some possibly overlapping structure.  Thus, we
assume that the output structure can be decomposed into $l$
substructures.  For example, this may be positions along a sequence,
so that the output space $\cY$ is decomposable along these
substructures: $\cY = \cY_1 \times \cdots \times \cY_l$. Here, $\cY_k$
is the set of possible labels (or classes) that can be assigned to
substructure $k$.

{\bf Loss functions}. We denote by
$\loss\colon \cY \times \cY \to \Rset_+$ a loss function measuring the
dissimilarity of two elements of the output space $\cY$.  We will
assume that the loss function $\loss$ is \emph{definite}, that is
$\loss(y, y') = 0$ iff $y = y'$. This assumption holds for all
loss functions commonly used in structured prediction.  A key aspect
of structured prediction is that the loss function can be decomposed
along the substructures $\cY_k$.  As an example, $\loss$ may be
the Hamming loss defined by
$\loss(y, y') = \frac{1}{l} \sum_{k = 1}^l 1_{y_k \neq y'_k}$ for
all $y = (y_1, \ldots, y_l)$ and $y' = (y'_1, \ldots, y'_l)$, with
$y_k, y'_k \in \cY_k$.  In the common case where $\cY$ is a set of
sequences defined over a finite alphabet, $\loss$ may be the
edit-distance, which is widely used in natural language and speech
processing applications, with possibly different costs associated to
insertions, deletions and substitutions.  $\loss$ may also be a loss
based on the negative inner product of the vectors of $n$-gram counts
of two sequences, or its negative logarithm. Such losses have been
used to approximate the BLEU score loss in machine translation. There are
other losses defined in computational biology based on various
string-similarity measures. Our theoretical analysis is general and
applies to arbitrary bounded and definite loss functions.

{\bf Scoring functions and factor graphs}.  We will adopt the common
approach in structured prediction where predictions are based on a
\emph{scoring function} mapping $\cX \times \cY$ to $\Rset$. Let $\cH$
be a family of scoring functions. For any $h \in \cH$, we denote by
$\hh$ the predictor defined by $h$: for any $x \in \cX$,
$\hh(x) = \argmax_{y \in \cY} h(x, y)$.

Furthermore, we will assume, as is standard in structured prediction,
that each function $h \in \cH$ can be decomposed as a sum.  We will
consider the most general case for such decompositions, which can be
made explicit using the notion of \emph{factor
  graphs}.\footnote{Factor graphs are typically used to indicate the
  factorization of a probabilistic model. We are not assuming
  probabilistic models, but they would be also captured by our general
  framework: $h$ would then be $\text{-}\log$ of a probability.}  A factor
graph $G$ is a tuple $G = (V, F, E)$, where $V$ is a set of variable
nodes, $F$ a set of factor nodes, and $E$ a set of undirected edges
between a variable node and a factor node. In our context, $V$ can be
identified with the set of substructure indices, that is
$V = \set{1, \ldots, l}$.

For any factor node $f$, denote by $\scrN(f) \subseteq V$ the set of
variable nodes connected to $f$ via an edge and define $\cY_f$ as
the substructure set cross-product $\cY_f = \prod_{k \in \scrN(f)} \cY_k$.
Then, $h$ admits the following decomposition as a sum of functions
$h_f$, each taking as argument an element of the input space
$x \in \cX$ and an element of $\cY_f$, $y_f \in \cY_f$:
\begin{equation}
\label{eq:decomposition}
h(x, y) = \sum_{f \in F} h_f(x, y_f).
\end{equation}
Figure~\ref{fig:factor_graph} illustrates this definition with two
different decompositions. More generally, we will consider the setting
in which a factor graph may depend on a particular example $(x_i, y_i)$:
$G(x_i, y_i) = G_i = ([l_i], F_i, E_i)$. A special case of this
setting is for example when the size $l_i$ (or length) of each example
is allowed to vary and where the number of possible labels $|\cY|$ is
potentially infinite.

\begin{figure}[t]
\centering
\begin{tabular}{c@{\hspace{1cm}}c}
\includegraphics[scale=0.4]{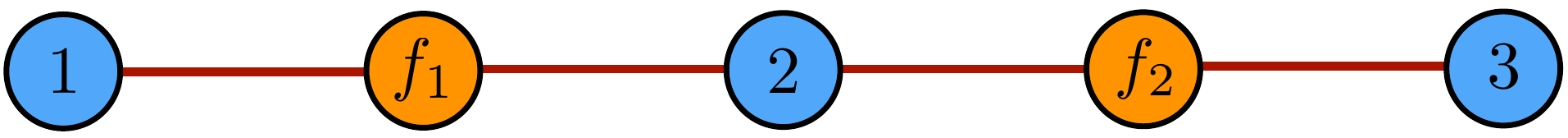} &
\includegraphics[scale=0.4]{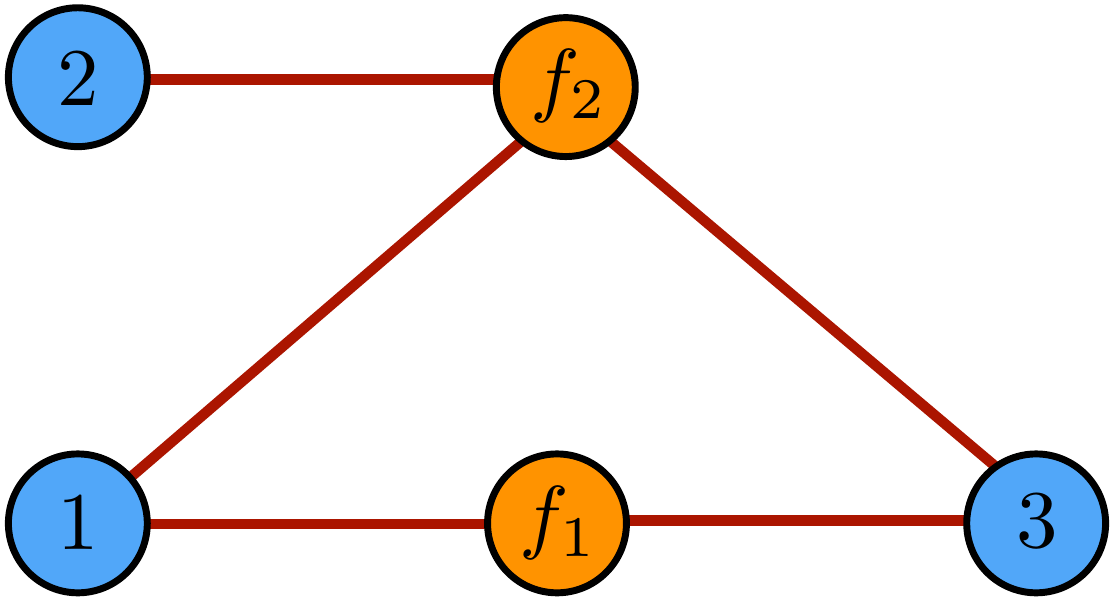} \\
(a) & (b)
\end{tabular}
\caption{Example of factor graphs. (a) Pairwise Markov network decomposition:
$h(x, y) =  h_{f_1}(x, y_1, y_2) + h_{f_2}(x, y_2, y_3)$
(b) Other decomposition $h(x, y) =
  h_{f_1}(x, y_1, y_3) + h_{f_2}(x, y_1, y_2, y_3)$.}
\label{fig:factor_graph}
\end{figure}

We present other examples of such hypothesis sets and their
decomposition in Section~\ref{sec:bounds}, where we discuss our
learning guarantees. Note that such hypothesis sets $\cH$ with an
additive decomposition are those commonly used in most structured
prediction algorithms \citep{TsochantaridisJoachimsHofmannAltun2005,
  TaskarGuestrinKoller2003,LaffertyMcCallumPereira2001}. This is
largely motivated by the computational requirement for efficient
training and inference. Our results, while very general,
further provide a statistical learning motivation for such
decompositions.

{\bf Learning scenario}.
We consider the familiar supervised learning scenario where the
training and test points are drawn i.i.d.\ according to some
distribution $\cD$ over $\cX \times \cY$. We will further adopt the
standard definitions of margin, generalization error and empirical
error. The margin $\rho_{h}(x, y)$ of a hypothesis $h$ for a labeled
example $(x, y) \in \cX \times \cY$ is defined by
\begin{equation}
\rho_{h}(x, y) = h(x, y) - \max_{y' \neq y} h(x, y').
\end{equation}
Let $S = ((x_1, y_1), \ldots, (x_m, y_m))$ be a training sample of
size $m$ drawn from $\cD^m$. We denote by $R(h)$ the generalization
error and by $\h R_S(h)$ the empirical error of $h$ over $S$:
\begin{align}
  R(h) = \E_{(x, y) \sim \cD} [\loss(\hh(x), y)] \quad\text{ and }\quad
  \h R_S(h) = \E_{(x, y) \sim S} [\loss(\hh(x), y)], 
\end{align}
where $\hh(x) = \argmax_y h(x, y)$ and where the notation
$(x, y) \!\!\sim\!\! S$ indicates that $(x, y)$ is drawn according
to the empirical distribution defined by $S$. The learning problem
consists of using the sample $S$ to select a hypothesis $h \in \cH$
with small expected loss $R(h)$.

Observe that the definiteness of the loss function implies, for
all $x \in \cX$, the following equality:
\begin{equation}
\label{eq:lossrho}
\loss(\hh(x), y) = \loss(\hh(x), y) \, 1_{\rho_h(x, y) \leq 0}.
\end{equation}
We will later use this identity in the derivation of surrogate loss
functions.

\section{General learning bounds for structured prediction}
\label{sec:bounds}

In this section, we present new learning guarantees for structured
prediction.  Our analysis is general and applies to the broad family
of definite and bounded loss functions described in the previous
section.  It is also general in the sense that it applies to general
hypothesis sets and not just sub-families of linear functions. For
linear hypotheses, we will give a more refined analysis that holds for
arbitrary norm-$p$ regularized hypothesis sets.

\ignore{
Our learning guarantees are based on a new notion of Rademacher
complexity, and they are somewhat finer than the previously best-known
bounds of \citet{TaskarGuestrinKoller2003}, which were given for the
special case of the Hamming loss with norm-2 regularization using
covering numbers. Our bounds are data-dependent, admit a favorable
dependency in terms of sparsity and the \emph{effective} factor graph,
and hold for general loss functions and regularizations. For the same
norm-2 regularization, they also admit a logarithmic term
improvement. We present below a more detailed comparison with previous
work.
}

The theoretical analysis of structured prediction is more complex than
for classification since, by definition, it depends on the properties
of the loss function and the factor graph. These attributes capture
the combinatorial properties of the problem which must be exploited
since the total number of labels is often exponential in the
size of that graph. To tackle this problem, we first introduce a new
complexity tool.

\subsection{Complexity measure}

A key ingredient of our analysis is a new data-dependent notion of
complexity that extends the classical Rademacher complexity. We define the
\emph{empirical factor graph Rademacher complexity} $\h \Rad^G_S(\cH)$
of a hypothesis set $\cH$ for a sample $S = (x_1, \ldots, x_m)$ and
factor graph $G$ as follows:
\begin{align*}
\h \Rad^G_S(\cH) = \frac{1}{m}
 \E_{\be}\Bigg[\sup_{h \in \cH}
\sum_{i = 1}^m \sum_{f \in F_i} \sum_{y \in \cY_f} \sqrt{|F_i|} \,
  \e_{i, f, y} \, h_f(x_i, y) \Bigg],
\end{align*}
where $\be = (\e_{i, f, y})_{i \in [m], f \in F_i, y \in \cY_f}$ and
where $\e_{i, f, y}$s are independent Rademacher random variables
uniformly distributed over $\set{\pm 1}$. The \emph{factor graph
  Rademacher complexity} of $\cH$ for a factor graph $G$ is defined as
the expectation:
$\Rad^G_m(\cH) = \E_{S \sim \cD^m} \mspace{-5mu} \big[ \h
\Rad^G_S(\cH) \big]$. It can be shown that the empirical factor graph 
Rademacher complexity is concentrated around its mean
(Lemma~\ref{lemma:empFGRCconcentration}).
The factor graph Rademacher complexity is a
natural extension of the standard Rademacher complexity to
vector-valued hypothesis sets (with one coordinate per factor in our
case). For binary classification, the factor graph and standard
Rademacher complexities coincide.  Otherwise, the factor graph
complexity can be upper bounded in terms of the standard one. As with
the standard Rademacher complexity, the factor graph Rademacher
complexity of a hypothesis set can be estimated from data in many
cases.  In some important cases, it also admits explicit upper bounds
similar to those for the standard Rademacher complexity but with an
additional dependence on the factor graph quantities. We will prove
this for several families of functions which are
commonly used in structured prediction
(Theorem~\ref{th:bound_linear}).

\subsection{Generalization bounds}

In this section, we present new margin bounds for structured
prediction based on the factor graph Rademacher complexity of $\cH$.
Our results hold both for the additive and the multiplicative
empirical margin losses defined below:
\begin{align}
\label{eq:add}
& \h R^\text{add}_{S, \rho}(h) =  \mspace{-2mu} \E_{(x, y) \sim S}
\bigg[\F \bigg(\max_{y' \neq y} \loss(y', y)  -  \tfrac{1}{\rho}
  \big[ h(x, y)  -  h(x, y') \big] \bigg) \bigg]\\
\label{eq:mult}
& \h R^\text{mult}_{S, \rho}(h) =  \mspace{-2mu} \E_{(x, y) \sim S}
\bigg[ \F \bigg( \max_{y' \neq y} \loss(y', y) \Big(1  -   \tfrac{1}{\rho}
 [h(x, y)  -  h(x, y')] \Big) \bigg) \bigg].
\end{align}
Here, $\F(r) = \min( M, \max(0, r))$ for all $r$, with
$M = \max_{y, y'} \loss(y, y')$.\ignore{
Lemma~\ref{lemma:surrogate},
which we present later in Section~\ref{sec:algo}, can be used to show
that the following holds:
\begin{align*}
 \loss(\hh(x), y) 
&\leq \max_{y' \neq y} \F( \loss(y', y)
                          (1 - \tfrac{1}{\rho} [h(x, y) - h(x, y')])) \\
&= \F(\max_{y' \neq y}\loss(y', y)
 (1 - \tfrac{1}{\rho} [h(x, y) - h(x, y')]))
\end{align*}
and similarly for the additive variant. Furthermore,}
As we show in Section~\ref{sec:algo}, convex upper bounds
on $\h R^\text{add}_{S, \rho}(h)$ and $\h R^\text{mult}_{S, \rho}(h)$
directly lead to many existing structured prediction algorithms. 
The following is our general data-dependent margin bound
for structured prediction.\\

\begin{theorem}
\label{th:bound}
Fix $\rho > 0$.
For any $\delta > 0$, with probability at least $1-\delta$ over the
draw of a sample $S$ of size $m$, the following holds
for all $h \in \cH$,
\begin{align*}
& R(h) \leq R^\text{add}_{\rho}(h) \leq \h R^\text{add}_{S, \rho}(h) + \frac{4\sqrt{2}}{\rho} \Rad^G_m(\cH)
+ M\sqrt{\frac{\log\frac{1}{\delta}}{2m}},\\
& R(h) \leq R^\text{mult}_{\rho}(h) \leq \h R^\text{mult}_{S, \rho}(h) + \frac{4\sqrt{2}M}{\rho} \Rad^G_m(\cH)
+ M\sqrt{\frac{\log\frac{1}{\delta}}{2m}}.
\end{align*}
\end{theorem}
The full proof of Theorem~\ref{th:bound} is given in
Appendix~\ref{app:bounds}. It is based on a new contraction lemma
(Lemma~\ref{lemma:contraction}) generalizing Talagrand's lemma that
can be of independent interest.\footnote{A result
  similar to Lemma~\ref{lemma:contraction} has also been recently
  proven independently in \citep{Maurer2016}.} We also present a more refined
contraction lemma (Lemma~\ref{lemma:contraction2}) that can be used to
improve the bounds of Theorem~\ref{th:bound}.  Theorem~\ref{th:bound}
is the first data-dependent generalization guarantee for structured
prediction with general loss functions, general hypothesis sets, and
arbitrary factor graphs for both multiplicative and additive
margins. We also present a version of this result with empirical
complexities as Theorem~\ref{th:bound_emp_complexity} in the supplementary material. We will compare these guarantees to known special cases below.

The margin bounds above can be extended to hold uniformly over
$\rho \in (0, 1]$ at the price of an additional term of the form
$\sqrt{\smash[b]{(\log \log_2 \tfrac{2}{\rho}) / m }}$ in the bound,
using known techniques (see for example
\citep{MohriRostamizadehTalwalkar2012}).

\ignore{
  Existing structured prediction bounds hold only for a specific
  decomposable loss such as the Hamming loss
  \citep{TaskarGuestrinKoller2003,CortesKuznetsovMohri2014,
    LondonHuangTaskarGetoor2013,Collins2001}, and do not cover many
  other natural loss functions, such as the edit-distance that are
  commonly used in applications. Also, they typically hold only for a
  particular factor graph structure.  The only exception is the work
  of \cite{McAllester2007} who provides PAC-Bayesian guarantees for
  arbitrary losses, in the special case of randomized algorithms using
  linear (count-based) hypotheses.  We will compare it to known
  special cases below.
}

The hypothesis set used by convex structured prediction algorithms
such as StructSVM \citep{TsochantaridisJoachimsHofmannAltun2005},
Max-Margin Markov Networks (M3N) \citep{TaskarGuestrinKoller2003} or
Conditional Random Field (CRF) \citep{LaffertyMcCallumPereira2001} is
that of linear functions. More precisely, let $\bPsi$ be a feature
mapping from $(\cX \times \cY)$ to $\Rset^{N}$ such that
$\bPsi(x, y) = \sum_{f\in F} \bPsi_f(x, y_f)$. For any $p$, define
$\cH_p$ as follows:
\begin{align*}
  \cH_p = \set{x \mapsto \bw \cdot \bPsi(x, y) \colon \bw \in
  \Rset^{N}, \| \bw \|_p \leq \Lambda_p}.
\end{align*}
Then, $\h \Rad^G_m(\cH_p)$ can be efficiently estimated using random
sampling and solving LP programs. Moreover, one can obtain explicit
upper bounds on $\h \Rad^G_m(\cH_p)$.  To simplify our presentation,
we will consider the case $p = 1, 2$, but our results can be extended
to arbitrary $p \geq 1$ and, more generally, to
arbitrary group norms.\\

\begin{theorem}
\label{th:bound_linear}
For any sample $S = (x_1, \ldots, x_m)$, the following upper bounds
hold for the empirical factor graph complexity of $\cH_1$ and $\cH_2$:
\begin{align*}
\h \Rad^G_S(\cH_1) \leq \frac{\Lambda_1 r_\infty}{m} \sqrt{s \log(2N)},\quad\quad \h \Rad^G_S(\cH_2) \leq \frac{\Lambda_2 r_2}{m}
\sqrt{\textstyle \sum_{i = 1}^m \sum_{f \in F_i} \sum_{y \in \cY_f} |F_i|},
\end{align*}
where $r_\infty = \max_{i, f, y} \| \Psi_f(x_i, y) \|_\infty$,
$r_2 = \max_{i, f, y} \| \Psi_f(x_i, y) \|_2$ and where $s$ is a
sparsity factor defined by
$s = \max_{j \in [1, N]} \sum_{i = 1}^m \sum_{f \in F_i} \sum_{y \in \cY_f}
|F_i| 1_{\Psi_{f, j}(x_i, y) \neq 0}$.
\end{theorem}
Plugging in these factor graph complexity upper bounds into
Theorem~\ref{th:bound} immediately yields explicit data-dependent
structured prediction learning guarantees for linear hypotheses with
general loss functions and arbitrary factor graphs (see
Corollary~\ref{cor:guarantee_linear}).  Observe that, in the worst case,
the sparsity factor can be bounded as follows:
\begin{align*}
s \leq \sum_{i = 1}^m \sum_{f \in F_i} \sum_{y \in \cY_f} |F_i|
\leq \sum_{i = 1}^m |F_i|^2 d_i \leq m \max_i |F_i|^2
d_i, 
\end{align*}
where $d_i = \max_{f \in F_i} |\cY_f|$.  Thus, the factor graph
Rademacher complexities of linear hypotheses in $\cH_1$ scale as
$O(\sqrt{\log(N) \max_i |F_i|^2 d_i / m})$. An important observation
is that $|F_i|$ and $d_i$ depend on the observed sample.
This shows that the \emph{expected
  size} of the factor graph is crucial for learning in this
scenario. This should be contrasted with other existing structured
prediction guarantees that we discuss below, which assume a fixed
upper bound on the size of the factor graph. Note that our result
shows that learning is possible even with an infinite set $\cY$. To
the best of our knowledge, this is the first learning guarantee for
learning with infinitely many classes.

Our learning guarantee for $\cH_1$ can additionally benefit from the
sparsity of the feature mapping and observed data. In particular, in
many applications, $\Psi_{f, j}$ is a binary indicator function that is
non-zero for a single $(x, y) \in \cX \times \cY_f$. For instance, in
NLP, $\Psi_{f, j}$ may indicate an occurrence of a certain $n$-gram in
the input $x_i$ and output $y_i$. In this case,
$s = \sum_{i = 1}^m |F_i|^2 \leq m \max_i |F_i|^2$ and the complexity
term is only in $O(\max_i |F_i| \sqrt{\log(N) / m})$, where $N$ may depend
linearly on $d_i$.

\subsection{Special cases and comparisons}
\label{sec:special}

\ignore{
We discuss some special cases and compare
our guarantees with those given by previous work.}

{\bf Markov networks}. For the pairwise Markov networks with a fixed
number of substructures $l$ studied by
\cite{TaskarGuestrinKoller2003}, our equivalent factor graph admits $l$
nodes, $|F_i| = l$, and the maximum size of $\cY_f$ is $d_i = k^2$ if
each substructure of a pair can be assigned one of $k$ classes. Thus, if we apply
Corollary~\ref{cor:guarantee_linear} with Hamming distance as our loss
function and divide the bound through by $l$, to
normalize the loss to interval $[0,1]$ as in \citep{TaskarGuestrinKoller2003},
we obtain the following explicit form of our guarantee for an
additive empirical margin loss, for all $h \in \cH_2$:
\begin{align*}
& R(h) \leq \h R^\text{add}_{S, \rho}(h) + \frac{4 \Lambda_2 r_2}{\rho} 
\sqrt{\frac{2 k^2}{m}}
+ 3 \sqrt{\frac{\log\frac{1}{\delta}}{2m}}.
\end{align*}
This bound can be further improved by eliminating the dependency on
$k$ using an extension of our contraction
Lemma~\ref{lemma:contraction} to $\|\cdot\|_{\infty,2}$ (see
Lemma~\ref{lemma:contraction2}).  The complexity term of
\cite{TaskarGuestrinKoller2003} is bounded by a quantity that varies as
$\tl O(\sqrt{\Lambda^2_2 q^2 r^2_2 /m})$, where $q$ is the maximal
out-degree of a factor graph. Our bound has the same dependence on
these key quantities, but with no logarithmic term in our case.  Note
that, unlike the result of \cite{TaskarGuestrinKoller2003}, our bound
also holds for general loss functions and different $p$-norm
regularizers. Moreover, our result for a multiplicative empirical
margin loss is new, even in this special case.

\ignore{
Furthermore, note that the particular contraction lemma
that we proved is for a $\| \cdot \|_{2}$-norm, but it can be extended
to an $\| \cdot \|_{p}$ result. Optimizing over $p$ then eliminates
the $k$-dependence in our bound.
The complexity term of \cite{TaskarGuestrinKoller2003} bounds depends
on the following $\tl O(\sqrt{\Lambda^2_2 q^2 /m})$. Our
bound has the same dependence on these key quantities, with no logarithmic
term in our case.   Note
that, unlike the result of \cite{TaskarGuestrinKoller2003}, our bound
also holds for general loss functions and different $p$-norm
regularizers. Moreover, our result for a multiplicative empirical
margin loss is new, even in this special case.}
\ignore{
This slightly improves upon the state-of-the-art bound of
\cite{TaskarGuestrinKoller2003} (by a logarithmic factor) for the same
setting: the complexity term in the learning guarantee of
\cite{TaskarGuestrinKoller2003} is of the form
$O(\Lambda_2 r_2 \sqrt{k^2\log(l k) / \rho^2 m})$. The comparison is
similar for other Markov networks (not necessarily pairwise).}

{\bf Multi-class classification}.  For standard (unstructured)
multi-class classification, we have $|F_i| = 1$ and $d_i = c$, where
$c$ is the number of classes. In that case, for linear hypotheses with norm-2 regularization, the
complexity term of our bound varies as
$O(\Lambda_2 r_2 \sqrt{c / \rho^2 m})$ 
(Corollary~\ref{cor:guarantee_linear_multiclass}).
This improves upon the best known general margin
bounds of \cite{KuznetsovMohriSyed2014}, who provide a guarantee that
scales linearly with the number of classes instead. Moreover, in the
special case where an individual $\bw_y$ is learned for each class
$y \in [c]$, we retrieve the recent favorable bounds given by
\cite{LeiDoganBinderKloft2015}, albeit with a somewhat simpler
formulation. In that case, for any $(x, y)$, all components of the
feature vector $\Psi(x, y)$ are zero, except (perhaps) for the $N$
components corresponding to class $y$, where $N$ is the dimension of
$\bw_y$.  In view of that, for example for a group-norm
$\| \cdot \|_{2, 1}$-regularization, the complexity term of our bound
varies as $O(\Lambda r \sqrt{(\log c) / \rho^2 m})$, which matches the
results of \cite{LeiDoganBinderKloft2015} with a logarithmic
dependency on $c$ (ignoring some complex exponents of $\log c$ in
their case). Additionally, note that unlike existing multi-class
learning guarantees, our results hold for arbitrary loss functions.
See Corollary~\ref{cor:guarantee_linear_multiclass_norm12} for further details.
Our sparsity-based bounds can also be used to give bounds
with logarithmic dependence on the number of classes when the features
only take values in $\set{0,1}$. Finally, using Lemma~\ref{lemma:contraction2}
instead of Lemma~\ref{lemma:contraction}, the dependency on the
number of classes can be further improved.

We conclude this section by observing that, since our guarantees are
expressed in terms of the average size of the factor graph over a
given sample, this invites us to search for a hypothesis set $\cH$ and
predictor $h \in \cH$ such that the tradeoff between the empirical
size of the factor graph and empirical error is optimal. In the next
section, we will make use of the recently developed principle of Voted
Risk Minimization (VRM) \citep{\ignore{CortesMohriSyed2014,
    KuznetsovMohriSyed2014,
    CortesKuznetsovMohriSyed2015,}CortesGoyalKuznetsovMohri2015} to
reach this objective.

\section{Voted Risk Minimization}
\label{sec:vrm}

In many structured prediction applications such as natural language
processing and computer vision, one may wish to exploit very rich
features. However, the use of rich families of hypotheses could lead
to overfitting.  In this section, we show that it may be possible to
use rich families in conjunction with simpler families, provided that
\emph{fewer} complex hypotheses are used (or that they are used with
less mixture weight). We achieve this goal by deriving learning
guarantees for ensembles of structured prediction rules that explicitly
account for the differing complexities between families.  This will
motivate the algorithms that we present in Section~\ref{sec:algo}.

Assume that we are given $p$ families $H_1, \ldots, H_p$ of functions
mapping from $\cX \times \cY$ to $\Rset$. Define the ensemble family
$\cF = \conv(\cup_{k = 1}^p H_k)$, that is the family of functions $f$
of the form $f = \sum_{t = 1}^T \alpha_t h_t$, where
$\Alpha = (\alpha_1, \ldots, \alpha_T)$ is in the simplex $\Delta$ and
where, for each $t \in [1, T]$, $h_t$ is in $H_{k_t}$ for some
$k_t \in [1, p]$.  We further assume that
$\Rad^G_m(H_1) \leq \Rad^G_m(H_2) \leq \ldots \leq \Rad^G_m(H_p)$. As
an example, the $H_k$s may be ordered by the size of the corresponding
factor graphs.

The main result of this section is a generalization of the VRM theory
to the structured prediction setting. The learning guarantees that we
present are in terms of upper bounds on $\h R^\text{add}_{S, \rho}(h)$
and $\h R^\text{mult}_{S, \rho}(h)$, which are defined as follows for
all $\tau \geq 0$:
\begin{align}
\label{eq:add_tau}
& \h R^\text{add}_{S, \rho, \tau}(h) =  \mspace{-2mu} \E_{(x, y) \sim S}
\bigg[\F \bigg(\max_{y' \neq y} \loss(y', y) + \tau -  \tfrac{1}{\rho}
  \big[ h(x, y)  -  h(x, y') \big] \bigg) \bigg]\\
\label{eq:mult_tau}
& \h R^\text{mult}_{S, \rho, \tau}(h) =  \mspace{-2mu} \E_{(x, y) \sim S}
\bigg[ \F \bigg( \max_{y' \neq y} \loss(y', y) \Big(1 + \tau - \tfrac{1}{\rho}
 [h(x, y)  -  h(x, y')] \Big) \bigg) \bigg].
\end{align}
Here, $\tau$ can be interpreted as a margin term that acts in conjunction with $\rho$. For simplicity, we assume in this section that $|\cY| = c < +\infty$.\\

\begin{theorem}
\label{th:main}
  Fix $\rho > 0$. For any $\delta > 0$, with probability at least
  $1 - \delta$ over the draw of a sample $S$ of size $m$, each of
  the following inequalities holds for all $f \in \cF$:
\begin{align*}
 &R(f) - \h R^\text{add}_{S, \rho, 1}(f)
\leq \frac{4\sqrt{2}}{\rho}
\sum_{t= 1}^T \alpha_t \Rad^G_m(H_{k_t}) + C(\rho, M, c, m, p), \\
 &R(f) - \h R^\text{mult}_{S, \rho, 1}(f)
\leq \frac{4\sqrt{2}M}{\rho}
\sum_{t= 1}^T \alpha_t \Rad^G_m(H_{k_t}) + C(\rho, M, c, m, p),
\end{align*}
where
$C(\rho, M, c, m, p) = \frac{2M}{\rho} \sqrt{\frac{\log p}{m}}
+ 3 M \sqrt{\Big\lceil \tfrac{4}{\rho^2} \log \big(\tfrac{c^2  \rho^2 m}{4 \log p}\big)
\Big\rceil \frac{\log p}{m} + \frac{\log \frac{2}{\delta}}{2m}}$.
\ignore{Thus, $R(f) \leq \h R_{S, \rho}(f) + 8 \frac{M}{\rho} \sqrt{\frac{ \pi \log(2pc \ov N)}{m}} \sum_{k = 1}^p
\|\bw^k\|_k r_k + O\left(\sqrt{\dfrac{\log
    p}{\rho^2 m} \log \Big[ \frac{\rho^2 M^2c^2 m}{4 \log p} \Big] }\right)$.}
\end{theorem}
The proof of this theorem crucially depends on the theory we developed in
Section~\ref{sec:bounds} and is given in
Appendix~\ref{app:bounds}. As with Theorem~\ref{th:bound}, we also
present a version of this result with empirical complexities as Theorem~\ref{th:main_emp_complexity}
in the supplementary material. The explicit dependence of this bound on
the parameter vector $\Alpha$ suggests that learning even with highly
complex hypothesis sets could be possible so long as the complexity
term, which is a weighted average of the factor graph complexities, is
not too large. The theorem provides a quantitative way of determining
the mixture weights that should be apportioned to each family.
Furthermore, the dependency on the number of distinct feature map
families $H_k$ is very mild and therefore suggests that a large number
of families can be used. These properties will be useful for 
motivating new algorithms for structured prediction.

\section{Algorithms}
\label{sec:algo}

In this section, we derive several algorithms for structured
prediction based on the VRM principle discussed in
Section~\ref{sec:vrm}. We first give general convex upper bounds
(Section~\ref{sec:surrogate}) on the structured prediction loss which
recover as special cases the loss functions used in StructSVM
\citep{TsochantaridisJoachimsHofmannAltun2005}, Max-Margin Markov
Networks (M3N) \citep{TaskarGuestrinKoller2003}, and Conditional Random
Field (CRF) \citep{LaffertyMcCallumPereira2001}. Next, we introduce 
a new algorithm, Voted Conditional Random Field (\VCRF) Section~\ref{sec:VCRF},
with accompanying experiments as proof of concept.
We also present another algorithm, Voted StructBoost (\VStructBoost), in
Appendix~\ref{sec:VStructBoost}.

\subsection{General framework for convex surrogate losses}
\label{sec:surrogate}

Given $(x, y) \in \cX \times \cY$, the mapping
$h \mapsto \loss(\hh(x), y)$ is typically not a convex function of
$h$, which leads to computationally hard optimization problems. This
motivates the use of convex surrogate losses. We first introduce a
general formulation of surrogate losses for structured prediction
problems.

\begin{lemma}
\label{lemma:surrogate}
For any $u \in \Rset_+$, let $\Phi_u\colon \Rset \to \Rset$ be an
upper bound on $v \mapsto u \I_{v \leq 0}$\ignore{ such that
$u \mapsto \Phi_u(v)$ is increasing for a fixed $v$}. Then, the
following upper bound holds for any $h \in \cH$ and
$(x, y) \in \cX \times \cY$,
\begin{align}
\label{eq:convex-bound}
\loss(\hh(x), y)
\leq \max_{y' \neq y} \Phi_{\loss( y', y)} (h(x, y) - h(x, y')).
\end{align}
\end{lemma}
The proof is given in Appendix~\ref{app:bounds}. This result defines a
general framework that enables us to straightforwardly recover many of
the most common state-of-the-art structured prediction algorithms via
suitable choices of $\Phi_u(v)$: (a) for
$\Phi_u(v) = \max(0, u(1 - v))$, the right-hand side of
\eqref{eq:convex-bound} coincides with the surrogate loss defining
StructSVM \citep{TsochantaridisJoachimsHofmannAltun2005}; (b) for
$\Phi_u(v) = \max(0, u - v)$, it coincides with the surrogate loss
defining Max-Margin Markov Networks (M3N) \citep{TaskarGuestrinKoller2003}
when using for $\loss$ the Hamming loss; and (c) for
$\Phi_u(v) = \log (1 + e^{u - v})$, it coincides with the surrogate
loss defining the Conditional Random Field (CRF)
\citep{LaffertyMcCallumPereira2001}.

Moreover, alternative choices of $\Phi_u(v)$ can help define new
algorithms. In particular, we will refer to the algorithm based on the
surrogate loss defined by $\Phi_u(v) = u e^{-v}$ as {\em
  \StructBoost}, in reference to the exponential loss used in
AdaBoost.  Another related alternative is based on the choice
$\Phi_u(v) = e^{u - v}$.  See Appendix~\ref{sec:VStructBoost}, for
further details on this algorithm.  In fact, for each $\Phi_u(v)$
described above, the corresponding convex surrogate is an upper bound
on either the multiplicative or additive margin loss introduced in
Section~\ref{sec:bounds}.  Therefore, each of these algorithms seeks a
hypothesis that minimizes the generalization bounds presented in
Section~\ref{sec:bounds}.  To the best of our knowledge, this
interpretation of these well-known structured prediction algorithms is
also new.  In what follows, we derive new structured prediction
algorithms that minimize finer generalization bounds presented in
Section~\ref{sec:vrm}.

\subsection{Voted Conditional Random Field (\VCRF)}
\label{sec:VCRF}

We first consider the convex surrogate loss based on
$\Phi_u(v) = \log (1 + e^{u - v})$, which corresponds to the loss
defining CRF models. Using the monotonicity of the logarithm
and upper bounding the maximum by a sum gives the following
upper bound on the surrogate loss holds:
\begin{align*}
\max_{ y' \neq y} \log(1 + e^{\loss(y, y') -
\bw \cdot (\bPsi(x, y) - \bPsi(x, y'))})
\ignore{
& =  \log(1 + \max_{y' \neq y} e^{\loss(y, y') -
\bw \cdot ( \bPsi(x, y) - \bPsi(x, y'))}) \\
& \leq
\log\Big(1 + \sum_{y' \neq y} e^{\loss(y, y') -
\bw \cdot ( \bPsi(x, y) - \bPsi(x, y')) }\Big),\\
}
& \leq
\log\Big(\sum_{y' \in \cY} e^{\loss(y, y') -
\bw \cdot ( \bPsi(x, y) - \bPsi(x, y')) }\Big),
\end{align*}
which, combined with
VRM principle leads to the following optimization problem:
\begin{align}
\label{eq:crf-opt1} 
\min_{\bw} \frac{1}{m} \sum_{i = 1}^m \log\bigg(\sum_{y \in \cY}
             e^{\loss(y, y_i) - \bw \cdot (\bPsi(x_i, y_i) - \bPsi(x_i, y)) }
             \bigg) + \sum_{k = 1}^p (\lambda r_k + \beta) \| \bw_k \|_1,
\end{align}
where $r_k = r_\infty |F(k)| \sqrt{\log N}$.  We refer to the learning
algorithm based on the optimization problem~\eqref{eq:crf-opt1} as
\VCRF. Note that for $\lambda = 0$, \eqref{eq:crf-opt1} coincides with
the objective function of $L_1$-regularized CRF.  Observe that we can
also directly use
$\max_{y' \neq y} \log(1 + e^{\loss(y, y') - \bw \cdot \d
  \bPsi(x, y, y')})$ or its upper bound
$\sum_{y' \neq y} \log(1 + e^{\loss(y, y') - \bw \cdot \d
  \bPsi(x, y, y')})$ as a convex surrogate.  We can similarly
derive an $L_2$-regularization formulation of the VCRF algorithm.  In
Appendix~\ref{app:opt}, we describe efficient algorithms for solving
the VCRF and VStructBoost optimization problems.\ignore{ Our
  algorithms can be further used with a large family of structured
  prediction losses that include the edit-distance and the $n$-gram
  loss function. To the best of our knowledge these results are also
  new.}

\section{Experiments}
\label{sec:experiments}

In Appendix~\ref{app:experiments}, we corroborate our theory by
reporting experimental results suggesting that the \VCRF\ algorithm
can outperform the CRF algorithm on a number of part-of-speech (POS)
datasets.

\section{Conclusion}

We presented a general theoretical analysis of structured prediction.
Our data-dependent margin guarantees for structured prediction can be
used to guide the design of new algorithms or to derive guarantees for
existing ones. Its explicit dependency on the properties of the factor
graph and on feature sparsity can help shed new light on the role
played by the graph and features in generalization. Our extension of
the VRM theory to structured prediction provides a new analysis of
generalization when using a very rich set of features, which is common
in applications such as natural language processing and leads to new
algorithms, VCRF and VStructBoost. Our experimental results for VCRF
serve as a proof of concept and motivate more extensive empirical
studies of these algorithms.

\subsubsection*{Acknowledgments} 

This work was partly funded by NSF CCF-1535987 and IIS-1618662 and NSF
GRFP DGE-1342536.

\newpage 
\bibliographystyle{abbrvnat} 
{\small \bibliography{vcrf}}

\newpage
\appendix

\section{Proofs}
\label{app:bounds}

This appendix section gathers detailed proofs of all of our main
results.  In Appendix~\ref{app:contraction}, we prove a contraction
lemma used as a tool in the proof of our general factor graph
Rademacher complexity bounds (Appendix~\ref{app:learning_bounds}). In
Appendix~\ref{app:VRM_bounds}, we further extend our bounds to the
Voted Risk Minimization setting.  Appendix~\ref{app:rademacher_bound}
gives explicit upper bounds on the factor graph Rademacher complexity
of several commonly used hypothesis sets.  In
Appendix~\ref{app:convex_surrogates}, we prove a general upper bound
on a loss function used in structured prediction in terms of a convex
surrogate.

\subsection{Contraction lemma}
\label{app:contraction}

The following contraction lemma will be a key tool used in the proofs
of our generalization bounds for structured prediction.\\

\begin{lemma}
\label{lemma:contraction}
Let $\cH$ be a hypothesis set of functions mapping $\cX$ to $\Rset^c$.
Assume that for all $i = 1, \ldots, m$, $\Psi_i: \Rset^c \to \Rset$ is $\mu_i$-Lipschitz
for $\Rset^c$ equipped with the 2-norm. That is:
\begin{equation*}
|\Psi_i(\bx') - \Psi_i(\bx) | 
\leq \mu_i \| \bx' -  \bx \|_2,
\end{equation*}
for all  $(\bx, \bx') \in (\Rset^c)^2$. 
  Then, for any sample $S$ of $m$ points
$x_1, \ldots, x_m  \in \cX$, the following inequality holds
\begin{equation}
\label{eq:contraction}
\frac{1}{m} \E_{\bsigma} \left[\sup_{\bh \in \cH} \sum_{i = 1}^m \sigma_i
\Psi_i(\bh(x_i)) \right] 
\leq \frac{ \sqrt{2}}{m} \E_{\be}
\left[\sup_{\bh \in \cH} \sum_{i = 1}^m \sum_{j = 1}^c \e_{ij} \, \mu_i h_j(x_i) \right],
\end{equation}
where $\be = (\e_{ij})_{i, j}$ and $\e_{ij}$s are independent
Rademacher variables uniformly distributed over $\set{\pm 1}$.
\end{lemma}

\begin{proof}
  Fix a sample $S = (x_1, \ldots, x_m)$. Then, we can rewrite the
  left-hand side of \eqref{eq:contraction} as follows:
\begin{equation*}
\frac{1}{m} \E_{\bsigma} \left[\sup_{\bh \in \cH} \sum_{i = 1}^m \sigma_i
\Psi_i(\bh(x_i)) \right]  = 
\frac{1}{m} \E_{\sigma_1, \ldots, \sigma_{m - 1}} \Big[ \E_{\sigma_m}\Big[
\sup_{\bh \in \cH} U_{m - 1}(\bh)  + \sigma_m \Psi_m(\bh(x_m)) \Big]
\Big] \,,
\end{equation*}
where $U_{m - 1}(\bh) = \sum_{i = 1}^{m - 1} \sigma_i
\Psi_i(\bh(x_i))$.
Assume that the suprema can be attained and let $\bh_1, \bh_2\in \cH$ be the
hypotheses satisfying
\begin{flalign*}
& U_{m - 1}(\bh_1) + \Psi_m(\bh_1(x_m))  =   \sup_{\bh \in \cH} U_{m - 1}(\bh) +  \Psi_m (\bh(x_m)) \\
& U_{m - 1}(\bh_2) - \Psi_m(\bh_2(x_m))  = \sup_{\bh \in \cH} U_{m - 1}(\bh)  -  \Psi_m(\bh(x_m)).
\end{flalign*}
When the suprema are not reached, a similar argument to what follows
can be given by considering instead hypotheses that are $\e$-close to
the suprema for any $\e > 0$.  By definition of expectation, since
$\sigma_m$ is uniformly distributed over $\set{\pm 1}$, we can write
\begin{align*}
& \mspace{-40mu} \E_{\sigma_m}\Big[ \sup_{\bh \in \cH} U_{m - 1}(\bh)  + \sigma_m \Psi_m(\bh(x_m)) \Big] \\
& = \frac{1}{2} \sup_{\bh \in \cH} U_{m - 1}(\bh)  +  \Psi_m(\bh(x_m))
+ \frac{1}{2} \sup_{\bh \in \cH} U_{m - 1}(\bh)  -  \Psi_m(\bh(x_m)) \\
& = \frac{1}{2} [U_{m - 1}(\bh_1) + \Psi_m(\bh_1(x_m))]
+ \frac{1}{2} [U_{m - 1}(\bh_2) - \Psi_m(\bh_2(x_m))].
\end{align*}
Next, using the $\mu_m$-Lipschitzness of $\Psi_m$ and the
Khintchine-Kahane inequality,
we can write
\begin{align*}
& \E_{\sigma_m}\Big[ \sup_{\bh \in \cH} U_{m - 1}(\bh)  + \sigma_m \Psi_m(\bh(x_m))
 \Big]\\
& \leq  \frac{1}{2} [U_{m - 1}(\bh_1) + U_{m - 1}(\bh_2) + \mu_m \| \bh_1(x_m) -
\bh_2(x_m) \|_2 ]\\
& \leq  \frac{1}{2} \Bigg[ U_{m - 1}(\bh_1) + U_{m - 1}(\bh_2) + \mu_m \sqrt{2}
\E_{\be_{m1}, \ldots, \be_{mc}}
\bigg[\Big| \sum_{j = 1}^c \e_{mj} \big( h_{1j}(x_m) -
h_{2j}(x_m) \big) \Big| \bigg] \Bigg].
\end{align*}
Now, let $\be_m$ denote
$(\be_{m1}, \ldots, \be_{mc})$ and let
$s(\be_m) \in \set{\pm 1}$ denote the sign of
$\sum_{j = 1}^c \e_{mj} \big( h_{1j}(x_m) - h_{2j}(x_m) \big)$. Then,
the following holds:
\begin{align*}
& \E_{\sigma_m}\Big[ \sup_{\bh \in \cH} U_{m - 1}(\bh)  + \sigma_m (\Psi_m
\circ h)(x_m) \Big]\\
& \leq \frac{1}{2} \E_{\be_m} \bigg[U_{m - 1}(\bh_1) + U_{m - 1}(\bh_2) + \mu_m \sqrt{2}
\Big| \sum_{j = 1}^c \e_{mj} \big( h_{1j}(x_m) -
h_{2j}(x_m) \big) \Big| \bigg]\\
& = \frac{1}{2} \E_{\be_m} \bigg[ U_{m - 1}(\bh_1) + \mu_m
  \sqrt{2} \, s(\be_m) \sum_{j = 1}^c \e_{mj} h_{1j}(x_m)  
 \\ & \qquad\qquad +  U_{m - 1}(\bh_2) - \mu_m
  \sqrt{2} \, s(\be_m) \sum_{j = 1}^c \e_{mj} h_{2j}(x_m)
  \bigg]\\
& \leq \frac{1}{2} \E_{\be_m} \bigg[\sup_{\bh \in \cH} \Big( U_{m - 1}(\bh) + \mu_m
  \sqrt{2} \, s(\be_m) \sum_{j = 1}^c \e_{mj} h_{j}(x_m) \Big)
 \\ & \qquad\qquad + \sup_{\bh \in \cH} \Big( U_{m - 1}(\bh) - \mu_m
  \sqrt{2} \, s(\be_m) \sum_{j = 1}^c \e_{mj} h_{j}(x_m) \Big)
  \bigg]\\
& = \E_{\be_m} \bigg[\E_{\sigma_m} \Big[ \sup_{\bh \in \cH} U_{m - 1}(\bh) + \mu_m
  \sqrt{2} \, \sigma_m \sum_{j = 1}^c \e_{mj} h_{j}(x_m) \Big]
  \bigg]\\
& = \E_{\be_m} \bigg[\sup_{\bh \in \cH} U_{m - 1}(\bh) + \mu_m
  \sqrt{2} \,  \sum_{j = 1}^c \e_{mj} h_{j}(x_m) \Big]
  \bigg],
\end{align*}
Proceeding in the same way for all other $\sigma_i$s ($i < m$)
completes the proof.
\end{proof}

\subsection{Contraction lemma for $\|\cdot\|_{\infty,2}$-norm}

In this section, we present an extension of the contraction Lemma~\ref{lemma:contraction}, that can be used to remove the dependency on the alphabet size in all of our bounds. 

\begin{lemma}
\label{lemma:contraction2}
Let $\cH$ be a hypothesis set of functions mapping $\cX \times [d]$ to
$\Rset^{c}$.
Assume that for all $i = 1, \ldots, m$, $\Psi_i$ is $\mu_i$-Lipschitz
for $\Rset^{c \times d}$ equipped with the norm-($\infty$, 2) for some $\mu_i >0$. That is
\begin{equation*}
|\Psi_i(\bx') - \Psi_i(\bx) | 
\leq \mu_i \| \bx' -  \bx \|_{\infty,2},
\end{equation*}
for all  $(\bx, \bx') \in (\Rset^{c \times d})^2$.
  Then, for any sample $S$ of $m$ points
$x_1, \ldots, x_m  \in \cX$, there exists a distribution
$\cU$ over $[d]^{c \times m}$ such that the following inequality holds:
\begin{equation}
\label{eq:contraction2}
\frac{1}{m} \E_{\bsigma} \left[\sup_{\bh \in \cH} \sum_{i = 1}^m \sigma_i
\Psi_i(\bh(x_i)) \right] 
\leq \frac{ \sqrt{2}}{m} \E_{\bup \sim \cU, \be}
\left[\sup_{\bh \in \cH} \sum_{i = 1}^m \sum_{j = 1}^c \e_{ij} \, \mu_i h_{j}(x_i, \upsilon_{mj}) \right],
\end{equation}
where $\be = (\e_{ij})_{i, j}$ and $\e_{ij}$s are independent
Rademacher variables uniformly distributed over $\set{\pm 1}$ and
$\bup = (\upsilon_{i, j})_{i,j}$ is a sequence of
random variables distributed according to $\cU$. 
Note that $\upsilon_{i, j}$s themselves
do not need to be independent.
\end{lemma}

\begin{proof}
  Fix a sample $S = (x_1, \ldots, x_m)$. Then, we can rewrite the
  left-hand side of \eqref{eq:contraction} as follows:
\begin{equation*}
\frac{1}{m} \E_{\bsigma} \left[\sup_{\bh \in \cH} \sum_{i = 1}^m \sigma_i
\Psi_i(\bh(x_i)) \right]  = 
\frac{1}{m} \E_{\sigma_1, \ldots, \sigma_{m - 1}} \Big[ \E_{\sigma_m}\Big[
\sup_{\bh \in \cH} U_{m - 1}(\bh)  + \sigma_m \Psi_m(\bh(x_m)) \Big]
\Big] \,,
\end{equation*}
where $U_{m - 1}(\bh) = \sum_{i = 1}^{m - 1} \sigma_i
\Psi_i(\bh(x_i))$.
Assume that the suprema can be attained and let $\bh_1, \bh_2\in \cH$ be the
hypotheses satisfying
\begin{flalign*}
& U_{m - 1}(\bh_1) + \Psi_m(\bh_1(x_m))  =   \sup_{\bh \in \cH} U_{m - 1}(\bh) +  \Psi_m (\bh(x_m)) \\
& U_{m - 1}(\bh_2) - \Psi_m(\bh_2(x_m))  = \sup_{\bh \in \cH} U_{m - 1}(\bh)  -  \Psi_m(\bh(x_m)).
\end{flalign*}
When the suprema are not reached, a similar argument to what follows
can be given by considering instead hypotheses that are $\e$-close to
the suprema for any $\e > 0$.  By definition of expectation, since
$\sigma_m$ is uniformly distributed over $\set{\pm 1}$, we can write
\begin{align*}
& \mspace{-40mu} \E_{\sigma_m}\Big[ \sup_{\bh \in \cH} U_{m - 1}(\bh)  + \sigma_m \Psi_m(\bh_1(x_m)) \Big] \\
& = \frac{1}{2} \sup_{\bh \in \cH} U_{m - 1}(\bh)  +  \Psi_m(\bh_1(x_m))
+ \frac{1}{2} \sup_{\bh \in \cH} U_{m - 1}(\bh)  -  \Psi_m(\bh(x_m)) \\
& = \frac{1}{2} [U_{m - 1}(\bh_1) + \Psi_m(\bh_1(x_m))]
+ \frac{1}{2} [U_{m - 1}(\bh_2) - \Psi_m(\bh_2(x_m))].
\end{align*}
Next, using the $\mu_m$-Lipschitzness of $\Psi_m$ and the
Khintchine-Kahane inequality,
we can write
\begin{align*}
& \E_{\sigma_m}\Big[ \sup_{\bh \in \cH} U_{m - 1}(\bh)  + \sigma_m (\Psi_m
\circ h)(x_m) \Big]\\
& \leq  \frac{1}{2} [U_{m - 1}(\bh_1) + U_{m - 1}(\bh_2) + \mu_m \| \bh_1(x_m) -
\bh_2(x_m) \|_{\infty,2} ]\\
& \leq  \frac{1}{2} \Bigg[ U_{m - 1}(\bh_1) + U_{m - 1}(\bh_2) + \mu_m \sqrt{2}
\E_{\be_{m1}, \ldots, \be_{mc}}
\bigg[\Big| \sum_{j = 1}^c \e_{mj} \| h_{1, j}(x_m, \cdot) -
h_{2, j}(x_m, \cdot) \|_\infty \Big| \bigg] \Bigg].
\end{align*}
Define the random variables $\upsilon_{mj} = \upsilon_{mj}(\bsigma) = \argmax_{k \in [d]} |h_{1,j}(x_m, k)-h_{2,j}(x_m, k)|$.

Now, let $\be_m$ denote
$(\be_{m1}, \ldots, \be_{mc})$ and let
$s(\be_m) \in \set{\pm 1}$ denote the sign of
$\sum_{j = 1}^c \e_{mj} \| h_{1, j}(x_m, \cdot) -
h_{2, j}(x_m, \cdot) \|_\infty $. Then,
the following holds:
\begin{align*}
& \E_{\sigma_m}\Big[ \sup_{\bh \in \cH} U_{m - 1}(\bh)  + \sigma_m (\Psi_m
\circ h)(x_m) \Big]\\
& \leq \frac{1}{2} \E_{\be_m} \bigg[U_{m - 1}(\bh_1) + U_{m - 1}(\bh_2) + \mu_m \sqrt{2}
\Big| \sum_{j = 1}^c \e_{mj} \| h_{1, j}(x_m, \cdot) -
h_{2, j}(x_m, \cdot) \|_\infty  \Big| \bigg]\\
& \leq \frac{1}{2} \E_{\be_m} \bigg[ U_{m - 1}(\bh_1) + U_{m - 1}(\bh_2)   
 \\ & \qquad\qquad + \mu_m
  \sqrt{2} \, s(\be_m) \sum_{j = 1}^c \e_{mj} |h_{1, j}(x_m, \upsilon_{mj}) -
   h_{2 , j}(x_m, \upsilon_{mj}) |
  \bigg]\\
& = \frac{1}{2} \E_{\be_m} \bigg[ U_{m - 1}(\bh_1) + U_{m - 1}(\bh_2)   
 \\ & \qquad\qquad + \mu_m
  \sqrt{2} \, s(\be_m) \sum_{j = 1}^c \e_{mj} (h_{1, j}(x_m, \upsilon_{mj}) -
   h_{2, j}(x_m, \upsilon_{mj}))
  \bigg]\\
& = \frac{1}{2} \E_{\be_m} \bigg[ U_{m - 1}(\bh_1) + \mu_m
  \sqrt{2} \, s(\be_m) \sum_{j = 1}^c \e_{mj} h_{1, j}(x_m, \upsilon_{mj})  
 \\ & \qquad\qquad +  U_{m - 1}(\bh_2) - \mu_m
  \sqrt{2} \, s(\be_m) \sum_{j = 1}^c \e_{mj} h_{2, j}(x_m, \upsilon_{mj})
  \bigg].
\end{align*}
After taking expectation over $\bup$, the rest of the proof proceeds the same way as the argument in
Lemma~\ref{lemma:contraction}:
\begin{align*}
& \frac{1}{2} \E_{\bup \sim \cU, \be_m} \bigg[ U_{m - 1}(\bh_1) + \mu_m
  \sqrt{2} \, s(\be_m) \sum_{j = 1}^c \e_{mj} h_{1, j}(x_m, \upsilon_{mj})  
 \\ & \qquad\qquad +  U_{m - 1}(\bh_2) - \mu_m
  \sqrt{2} \, s(\be_m) \sum_{j = 1}^c \e_{mj} h_{2, j}(x_m, \upsilon_{mj})
  \bigg]\\
& \leq \frac{1}{2} \E_{\bup \sim \cU, \be_m} \bigg[\sup_{\bh \in \cH} \Big( U_{m - 1}(\bh) + \mu_m
  \sqrt{2} \, s(\be_m) \sum_{j = 1}^c \e_{mj} h_{ j}(x_m, \upsilon_{mj}) \Big)
 \\ & \qquad\qquad + \sup_{\bh \in \cH} \Big( U_{m - 1}(\bh) - \mu_m
  \sqrt{2} \, s(\be_m) \sum_{j = 1}^c \e_{mj} h_{j}(x_m, \upsilon_{mj}) \Big)
  \bigg]\\
& = \E_{\bup \sim \cU, \be_m} \bigg[\E_{\sigma_m} \Big[ \sup_{\bh \in \cH} U_{m - 1}(\bh) + \mu_m
  \sqrt{2} \, \sigma_m \sum_{j = 1}^c \e_{mj} h_{ j}(x_m, \upsilon_{mj}) \Big]
  \bigg]\\
& = \E_{\bup \sim \cU, \be_m} \bigg[\sup_{\bh \in \cH} U_{m - 1}(\bh) + \mu_m
  \sqrt{2} \,  \sum_{j = 1}^c \e_{mj} h_{j}(x_m, \upsilon_{mj}) \Big]
  \bigg],
\end{align*}
Proceeding in the same way for all other $\sigma_i$s ($i < m$)
completes the proof.
\end{proof}

\subsection{General structured prediction learning bounds}
\label{app:learning_bounds}

In this section, we give the proof of several general structured
prediction bounds in terms of the notion of factor graph Rademacher
complexity. We will use the additive and multiplicative margin losses
of a hypothesis $h$, which are the population versions of the
empirical margin losses we introduced in \eqref{eq:add} and
\eqref{eq:mult} and are defined as follows:
\begin{align*}
& R^\text{add}_\rho(h) =  \mspace{-2mu} \E_{(x, y) \sim \cD}
\bigg[\F \bigg(\max_{y' \neq y} \loss(y', y)  -  \tfrac{1}{\rho}
  \big[ h(x, y)  -  h(x, y') \big] \bigg) \bigg]\\
& R^\text{mult}_\rho(h) =  \mspace{-2mu} \E_{(x, y) \sim \cD}
\bigg[ \F \bigg( \max_{y' \neq y} \loss(y', y) \Big(1  -   \tfrac{1}{\rho}
 [h(x, y)  -  h(x, y')] \Big) \bigg) \bigg].
\end{align*}
The following is our general margin bound for structured prediction.\\

\begin{reptheorem}{th:bound}
Fix $\rho > 0$.
For any $\delta > 0$, with probability at least $1-\delta$ over the
draw of a sample $S$ of size $m$, the following holds
for all $h \in \cH$,
\begin{align*}
&R(h) \leq R^\text{add}_{\rho}(h) \leq \h R^\text{add}_{S, \rho}(h) + \frac{4\sqrt{2}}{\rho} \Rad^G_m(\cH)
+ M\sqrt{\frac{\log\frac{1}{\delta}}{2m}},\\
&R(h) \leq R^\text{mult}_{\rho}(h) \leq \h R^\text{mult}_{S, \rho}(h) + \frac{4\sqrt{2}M}{\rho} \Rad^G_m(\cH)
+ M\sqrt{\frac{\log\frac{1}{\delta}}{2m}}.
\end{align*}
\end{reptheorem}

\begin{proof}
Let $\Phi_u(v) = \F(u-\frac{v}{\rho})$, where $\F(r) = \min(M, \max(0, r))$.
Observe that for any $u \in [0,M]$,
  $u 1_{v \leq 0} \leq \Phi_u(v)$ for all $v$. \ignore{Furthermore, $\Phi_u(v)$
  is an increasing function of $u$ for any given $v$.} Therefore,
by Lemma~\ref{lemma:surrogate} and monotonicity of $\F$,
\begin{align*}
R(h)
& \leq
\E_{(x, y) \sim \cD}[\max_{y' \neq y}
\Phi_{\loss(y',y)} (h(x, y) - h(x, y'))] \\
& =
\E_{(x, y) \sim \cD}\Bigg[\F\Bigg(
\max_{y' \neq y}\Big(\loss(y',y) - \frac{h(x, y) - h(x, y')}{\rho}
\Big)\Bigg)\Bigg] \\
& = R^\text{add}_\rho(h).
\end{align*} 
Define
\begin{align*}
& \cH_0 = \Bigg\{(x, y) \mapsto \F\Big(
\max_{y' \neq y}\Big(\loss(y',y) - \frac{h(x, y) - h(x, y')}{\rho}\Big)\Big)
\colon h \in \cH\Bigg\}, \\
&\cH_1 = \Bigg\{(x, y) \mapsto 
\max_{y' \neq y}\Big(\loss(y',y) - \frac{h(x, y) - h(x, y')}{\rho}\Big)
\colon h \in \cH\Bigg\}.
\end{align*}
By standard Rademacher complexity bounds 
(\citeapp{KoltchinskiiPanchenko2002}), for any $\delta > 0$, with
probability at least $1 - \delta$, the following inequality holds for all
$h \in \cH$:
\begin{align*}
R^\text{add}_\rho(h) \leq \h R^\text{add}_{S,\rho}(h) + 2 \Rad_m(\cH_0) +
M\sqrt{\frac{\log\frac{1}{\d}}{2m}},
\end{align*}
where $\Rad_m(\cH_0)$ is the Rademacher complexity of the
family $\cH_0$:
\begin{align*}
\Rad_m(\cH_0) = \frac{1}{m}\E_{S \sim \cD^m}\E_{\boldsymbol{\sigma}}
\Bigg[ \sup_{h \in \cH} \sum_{i=1}^m \sigma_i
\F\Big(\max_{y' \neq y_i}\Big(\loss(y',y_i) - \frac{h(x_i,y_i) -
h(x_i,y')}{\rho}\Big)\Big)
 \Bigg]
\end{align*}
and where $\bsigma = (\sigma_1, \ldots, \sigma_m)$ with $\sigma_i$s
independent Rademacher random variables uniformly distributed over
$\set{\pm 1}$.  
Since $\F$ is $1$-Lipschitz, by Talagrand's contraction
lemma (\citeapp{LedouxTalagrand1991,MohriRostamizadehTalwalkar2012}), we
have $\h \Rad_S(\cH_0) \leq \h \Rad_S(\cH_1)$. By taking an expectation over $S$,
  this inequality carries over to the true Rademacher complexities as well. Now, observe
that by the sub-additivity of the supremum, the following holds:
\begin{align*}
\h \Rad_S(\cH_1)
\leq &
\frac{1}{m}\E_{\boldsymbol{\sigma}}
\Bigg[ \sup_{h \in \cH} \sum_{i=1}^m \sigma_i
\max_{y' \neq y_i}\Big(\loss(y',y_i) + \frac{h(x_i,y')}{\rho}\Big)
 \Bigg]
\\ &+
\frac{1}{m}\E_{\boldsymbol{\sigma}}
\Bigg[ \sup_{h \in \cH} \sum_{i=1}^m \sigma_i
\frac{h(x_i,y_i)}{\rho}
 \Bigg],
\end{align*}
where we also used for the last term the fact that $-\sigma_i$ and
$\sigma_i$ admit the same distribution.  We use
Lemma~\ref{lemma:contraction} to bound each of the two terms appearing
on the right-hand side separately. To do so, we we first show the
Lipschitzness of
$h \mapsto \max_{y' \neq y_i}\Big(\loss(y',y_i) +
\frac{h(x_i,y')}{\rho}\Big)$.
Observe that the following chain of inequalities holds for any
$h, \tl h \in \cH$:
\begin{align*}
\Bigg|\max_{y \neq y_i}
\Bigg(\loss(y,y_i) + \frac{h(x_i, y)}{\rho} \Bigg)
&-\max_{y \neq y_i}
\Bigg(\loss(y,y_i) +\frac{\tl  h(x_i, y)}{\rho} \Bigg)
\Bigg|\\
&\leq
 \frac{1}{\rho} \max_{y \neq y_i}
\Big| h(x_i, y) - \tl h(x_i, y) \Big|  \\
&\leq
 \frac{1}{\rho} \max_{y \in \cY}
\Big| h(x_i, y) - \tl h(x_i, y) \Big|  \\
&=
 \frac{1}{\rho} \max_{y \in \cY}
\Big| \sum_{f \in F_i} (h_f(x_i, y_f) - \tl h_f(x_i, y_f)) \Big|  \\
&\leq
 \frac{1}{\rho} \sum_{f \in F_i} \max_{y \in \cY}
\Big| (h_f(x_i, y_f) - \tl h_f(x_i, y_f)) \Big|  \\
&=
 \frac{1}{\rho} \sum_{f \in F_i} \max_{y \in \cY_f}
\Big| (h_f(x_i, y) - \tl h_f(x_i, y)) \Big|  \\
&\leq
 \frac{\sqrt{|F_i|}}{\rho}  \sqrt{ \sum_{f \in F_i} \Big[\max_{y \in \cY_f}
| (h_f(x_i, y) - \tl h_f(x_i, y)) |\Big]^2 }  \\
&=
 \frac{\sqrt{|F_i|}}{\rho}  \sqrt{ \sum_{f \in F_i} \max_{y \in \cY_f}
| (h_f(x_i, y) - \tl h_f(x_i, y)) |^2 }  \\
&\leq
 \frac{\sqrt{|F_i|}}{\rho}  \sqrt{ \sum_{f \in F_i} \sum_{y \in \cY_f}
| (h_f(x_i, y) - \tl h_f(x_i, y)) |^2 }.
\end{align*}
We can therefore apply Lemma~\ref{lemma:contraction}, which yields
\begin{align*}
& \mspace{-40mu} \frac{1}{m}\E_{\boldsymbol{\sigma}}
\Bigg[ \sup_{h \in \cH} \sum_{i=1}^m \sigma_i
\max_{y' \neq y_i}\Big(\loss(y',y_i) + \frac{h(x_i,y')}{\rho}\Big)
 \Bigg] \\
& \leq
\frac{\sqrt{2}}{m}
\E_{\boldsymbol{\epsilon}}
\Bigg[ \sup_{h \in \cH} \sum_{i=1}^m \sum_{f \in F_i} \sum_{y \in \cY_f}
\e_{i,f,y} \frac{\sqrt{|F_i|}}{\rho} h_f(x_i, y) \Bigg] 
= \frac{\sqrt{2}}{\rho} \h \Rad^G_S(\cH).
\end{align*}
Similarly, for the second term, observe that the following Lipschitz
property holds:
\begin{align*}
\Big|\frac{h(x_i, y_i)}{\rho} - \frac{\tl h(x_i, y_i)}{\rho} \Big|
&\leq
\frac{1}{\rho} \max_{y \in \cY}
\Big| h(x_i, y) - \tl h(x_i, y) \Big| \\
&\leq
 \frac{\sqrt{|F_i|}}{\rho}  \sqrt{ \sum_{f \in F_i} \sum_{y \in \cY}
| (h_f(x_i, y) - \tl h_f(x_i, y)) |^2 }.
\end{align*}
We can therefore apply Lemma~\ref{lemma:contraction} and obtain the following:
\begin{align*}
 \frac{1}{m}\E_{\boldsymbol{\sigma}}
\Bigg[ \sup_{h \in \cH} \sum_{i=1}^m \sigma_i
\frac{h(x_i,y_i)}{\rho}
 \Bigg] 
& \leq
\frac{\sqrt{2}}{m}
\E_{\boldsymbol{\epsilon}}
\Bigg[ \sup_{h \in \cH} \sum_{i=1}^m \sum_{f \in F_i} \sum_{y \in \cY_f}
\e_{i,f,y} \frac{\sqrt{|F_i|}}{\rho} h_f(x_i, y) \Bigg] 
= \frac{\sqrt{2}}{\rho} \h \Rad^G_S(\cH).
\end{align*}
Taking the expectation over $S$ of the two inequalities
shows that $\Rad_m(\cH_1) \leq \frac{2\sqrt{2}}{\rho} \Rad^G_m(\cH)$,
which completes the proof of the first statement.

The second statement can be proven in a similar way with
$\Phi_u(v) = \F(u(1-\frac{v}{\rho}))$. In particular, by standard
Rademacher complexity bounds, McDiarmid's inequality, and Talagrand's contraction lemma, we can
write
\begin{align*}
R^\text{mult}_\rho(h) \leq \h R^\text{mult}_{S,\rho}(h) + 2 \Rad_m(\tl \cH_1) +
M\sqrt{\frac{\log\frac{1}{\d}}{2m}},
\end{align*}
where
\begin{align*}
\tl \cH_1 = \Bigg\{(x, y) \mapsto 
\max_{y' \neq y}\loss(y',y)\Big(1 - \frac{h(x, y) - h(x, y')}{\rho}\Big)
\colon h \in \cH\Bigg\}.
\end{align*}
We observe that the following inequality holds:
\begin{align*}
\Bigg|\max_{y \neq y_i}
\loss(y,y_i)\Bigg(1 - \frac{h(x_i, y_i) - h(x_i, y)}{\rho} \Bigg)
&- \max_{y \neq y_i}
\loss(y,y_i)\Bigg( 1 - \frac{\tl  h(x_i, y_i) - \tl  h(x_i, y)}{\rho} \Bigg)
\Bigg|\\
&\leq
 \frac{2M}{\rho} \max_{y \in \cY}
\Big| h(x_i, y) - \tl h(x_i, y) \Big|.
\end{align*}
Then, the rest of the proof follows from Lemma~\ref{lemma:contraction}
as in the previous argument.
\end{proof}

In the proof above, we could have applied McDiarmid's inequality
to bound the Rademacher complexity of $\cH_0$ by its empirical counterpart
at the cost of slightly increasing the exponential concentration term:
\begin{align*}
R^\text{add}_\rho(h) \leq \h R^\text{add}_{S,\rho}(h) + 2 \h \Rad_S(\cH_0) +
3M\sqrt{\frac{\log\frac{1}{\d}}{2m}}.
\end{align*}
Since Talagrand's contraction lemma holds for empirical Rademacher complexities and 
the remainder of the proof involves bounding the empirical Rademacher
complexity of $\cH_1$ before taking an expectation over the sample at the end, we can
apply the same arguments without the final expectation to arrive at the 
following analogue of Theorem~\ref{th:bound} in terms of
empirical complexities: 
\begin{theorem}
  \label{th:bound_emp_complexity}
Fix $\rho > 0$.
For any $\delta > 0$, with probability at least $1-\delta$ over the
draw of a sample $S$ of size $m$, the following holds
for all $h \in \cH$,
\begin{align*}
&R(h) \leq R^\text{add}_{\rho}(h) \leq \h R^\text{add}_{S, \rho}(h) + \frac{4\sqrt{2}}{\rho} \h \Rad^G_S(\cH)
+ 3M\sqrt{\frac{\log\frac{1}{\delta}}{2m}},\\
&R(h) \leq R^\text{mult}_{\rho}(h) \leq \h R^\text{mult}_{S, \rho}(h) + \frac{4\sqrt{2}M}{\rho} \h \Rad^G_S(\cH)
+ 3M\sqrt{\frac{\log\frac{1}{\delta}}{2m}}.
\end{align*}
\end{theorem}
This theorem will be useful for many of our applications, which are
based on bounding the empirical factor graph Rademacher complexity for
different hypothesis classes.

\ignore{
Let $H$ be a family of function mapping from $\cX$ to $\Rset^c$.  We
will denote by $h_{j}$, $j \in [1, c]$ the components of an element
$\bh \in \H$.  Here, we will say that a function
$\Phi\colon \Rset^c \to \Rset$ is $\mu$-Lipschitz for some $\mu > 0$
if it is $\mu$-Lipschitz for $\Rset^c$ equipped with the norm-2
metric:
\begin{equation*}
\forall (\bx, \bx') \in (\Rset^c)^2, 
|\Phi(\bx') - \Phi(\bx) | 
\leq \| \bx' -  \bx \|_2.
\end{equation*}}

\subsection{Concentration of the empirical factor graph Rademacher
  complexity}

In this section, we show that, as with the standard notion of
Rademacher complexity, the empirical factor graph Rademacher
complexity also concentrates around its mean.
\begin{lemma}
\label{lemma:empFGRCconcentration}
Let $\cH$ be a family of scoring functions mapping
$\cX \times \cY \to \Rset$ bounded by a constant $C$.  Let $S$ be a training sample of size $m$
drawn i.i.d.\ according to some distribution $\cD$ on
$\cX \times \cY$, and let $\cD_\cX$ be the marginal distribution on
$\cX$. For any point $x \in \cX$, let $F_x$ denote its associated set
of factor nodes.  Then, with probability at least $1 - \delta$ over
the draw of sample $S \sim \cD^m$,
\begin{equation*}
\left| \h \Rad^G_S(\cH) - \Rad^G_m(\cH) \right| \leq 2 C \sup_{x \in
  \text{supp}(\cD_\cX)} \sum_{f \in F_x} |\cY_f| \sqrt{|F_x|}
\sqrt{\frac{\log \frac{2}{\delta}}{2m}}.
\end{equation*}
\end{lemma}
\begin{proof}
 Let $S = (x_1 ,x_2 ,\ldots, x_m)$ and $S' = (x_1',x_2', \ldots, x_m')$ be two samples differing by one point $x_j$ and $x_j'$ (i.e. $x_i = x_i'$ for $i \neq j$).  Then
\begin{align*}
\h \Rad^G_S(\cH) - \h \Rad^G_{S'}(\cH) 
& \leq  \frac{1}{m} \E_{\be}\Bigg[\sup_{h \in \cH}
\sum_{i = 1}^m \sum_{f \in F_i} \sum_{y \in \cY_f} \sqrt{|F_i|} \,
  \e_{i, f, y} \, h_f(x_i, y) \Bigg]\\
& \quad - \frac{1}{m}
 \E_{\be}\Bigg[\sup_{h \in \cH}
\sum_{i = 1}^m \sum_{f \in F_i} \sum_{y \in \cY_f} \sqrt{|F_i|} \,
  \e_{i, f, y} \, h_f(x_i', y) \Bigg] \\
& = \frac{1}{m}
 \E_{\be}\Bigg[\sup_{h \in \cH}
\sum_{f \in F_{x_j}} \sum_{y \in \cY_f} \sqrt{|F_j|} \,
  \e_{j, f, y} \, h_f(x_j, y) \\
& \qquad \qquad - 
\sum_{f' \in F_{x_j'}} \sum_{y \in \cY_f'} \sqrt{|F_{x_j'}|} \,
  \e_{j, f', y} \, h_f(x_j', y) \Bigg] \\
& \leq \frac{2}{m} \sup_{x \in \text{supp}(\cD_\cX)} \sup_{h \in \cH} \sum_{f \in F_x} \sum_{y \in \cY_f} \sqrt{|F_x|} |h_f(x,y)|.  
\end{align*}
The same upper bound also holds for
$\h \Rad^G_{S'}(\cH) - \h \Rad^G_S(\cH)$. The result now follows from
McDiarmid's inequality.
\end{proof}

\subsection{Bounds on the factor graph Rademacher complexity}
\label{app:rademacher_bound}

The following lemma is a standard bound on the expectation of the
maximum of $n$ zero-mean bounded random variables, which will be used
in the proof of our bounds on factor graph Rademacher complexity.\\

\begin{lemma}
\label{lemma:MaxConcentration}
Let $X_1 \ldots X_n$ be $n \geq 1$ real-valued random variables such
that for all $j \in [1, n]$, $X_j = \sum_{i = 1}^{m_j} Y_{ij}$ where, for
each fixed $j \in [1, n]$, $Y_{ij}$ are independent zero mean random
variables with $|Y_{ij}| \leq t_{ij}$. Then,
the following inequality holds:
\begin{equation*}
\E\Big[\max_{j \in [1, n]} X_j \Big] \leq  t \sqrt{2\log n},
\end{equation*}
with $t =  \sqrt{\max_{j \in [1,n]} \sum_{i = 1}^{m_j} t_{ij}^2}$.
\end{lemma}

The following are upper bounds on the factor graph Rademacher
complexity for $\cH_1$ and $\cH_2$, as defined in Section~\ref{sec:bounds}. Similar guarantees can be given
for other hypothesis sets $\cH_p$ with $p > 1$.\\

\begin{reptheorem}{th:bound_linear}
For any sample $S = (x_1, \ldots, x_m)$, the following upper bounds
hold for the empirical factor graph complexity of $\cH_1$ and $\cH_2$:
\begin{align*}
\h \Rad^G_S(\cH_1) \leq \frac{\Lambda_1 r_\infty}{m} \sqrt{s \log(2N)}, \quad\quad \h \Rad^G_S(\cH_2) \leq \frac{\Lambda_2 r_2}{m}
\sqrt{\textstyle \sum_{i = 1}^m \sum_{f \in F_i} \sum_{y \in \cY_f} |F_i|},
\end{align*}
where $r_\infty = \max_{i, f, y} \| \Psi_f(x_i, y) \|_\infty$,
$r_2 = \max_{i, f, y} \| \Psi_f(x_i, y) \|_2$ and where $s$ is a
sparsity factor defined by
$s = \max_{j \in [1, N]} \sum_{i = 1}^m \sum_{f \in F_i} \sum_{y \in \cY_f}
|F_i| 1_{\Psi_j(x_i, y) \neq 0}$.
\end{reptheorem}

\begin{proof}
  By definition of the dual norm and
  Lemma~\ref{lemma:MaxConcentration} (or Massart's lemma), the
  following holds:
\begin{align*}
m \h \Rad^G_S(\cH_1)
& = 
\E_\be \bigg[\sup_{\|\bw\|_1 \leq \Lambda_1} \bw \cdot   
\sum_{i = 1}^m \sum_{f \in F_i} \sum_{y \in \cY_f} \sqrt{|F_i|}
\e_{i, f, y}  \bPsi_f(x_i, y) \bigg] \\
& =
\Lambda_1 \E_\be \bigg[\bigg \|
\sum_{i = 1}^m \sum_{f \in F_i} \sum_{y \in \cY_f} \sqrt{|F_i|}
\e_{i, f, y}  \bPsi_f(x_i, y)
\bigg\|_\infty \bigg] \\
& =
\Lambda_1 \E_\be \left[ \max_{j\in[1,N], \sigma \in \set{-1, +1}}  
  \sigma \sum_{i = 1}^m \sum_{f \in F_i} \sum_{y \in \cY_f} \sqrt{|F_i|}
\e_{i, f, y} \Psi_{f,j}(x_i, y)  \right]\\
& = \Lambda_1 \E_\be \left[ \max_{j\in[1,N], \sigma \in \set{-1, +1}}  
  \sigma  \sum_{i = 1}^m \sum_{f \in F_i} \sum_{y \in \cY_f} \sqrt{|F_i|}
\e_{i, f, y} \Psi_{f,j}(x_i, y) 1_{\Psi_{f,j}(x_i, y) \neq 0}  \right]\\
& \leq \Lambda_1 \sqrt{2 \Big(\max_{j \in [1, N]}
\sum_{i = 1}^m \sum_{f \in F_i} \sum_{y \in \cY_f} |F_i| 
  1_{\Psi_j(x_i, y) \neq 0} \Big) r_\infty^2 \log (2 N)}\\
& = \Lambda_1 r_\infty \sqrt{2 s \log (2 N)},
\end{align*}
which completes the proof of the first statement.  The second
statement can be proven in a similar way using the the definition of
the dual norm and Jensen's inequality:
\begin{align*}
m \h \Rad^G_S(\cH_2)
& = 
\E_\be \bigg[\sup_{\|\bw\|_2 \leq \Lambda_2} \bw \cdot   
\sum_{i = 1}^m \sum_{f \in F_i} \sum_{y \in \cY_f} \sqrt{|F_i|}
\e_{i, f, y}  \bPsi_f(x_i, y) \bigg] \\
& =
\Lambda_2 \E_\be \bigg[\bigg \|
\sum_{i = 1}^m \sum_{f \in F_i} \sum_{y \in \cY_f} \sqrt{|F_i|}
\e_{i, f, y}  \bPsi_f(x_i, y)
\bigg\|_2 \bigg] \\
& =
\Lambda_2 \Bigg(\E_\be \bigg[\bigg \|
\sum_{i = 1}^m \sum_{f \in F_i} \sum_{y \in \cY_f} \sqrt{|F_i|}
\e_{i, f, y}  \bPsi_f(x_i, y)
\bigg\|_2^2 \bigg] \Bigg)^{\frac{1}{2}} \\
& =
\Lambda_2 \Bigg(
\sum_{i = 1}^m \sum_{f \in F_i} \sum_{y \in \cY_f} |F_i|
 \| \bPsi_f(x_i, y) \|_2^2  \Bigg)^{\frac{1}{2}} \\
& \leq \Lambda_2 r_2
\sqrt{\sum_{i = 1}^m \sum_{f \in F_i} \sum_{y \in \cY_f} |F_i|},
\end{align*}
which concludes the proof.
\end{proof}

\subsection{Learning guarantees for structured prediction with linear hypotheses}

The following result is a direct consequence of 
Theorem~\ref{th:bound_emp_complexity} and Theorem~\ref{th:bound_linear}.\\

\begin{corollary}
\label{cor:guarantee_linear}
Fix $\rho > 0$.  For any $\delta > 0$, with probability at least
$1 - \delta$ over the draw of a sample $S$ of size $m$, the following
holds for all $h \in \cH_1$,
\begin{align*}
& R(h) \leq \h R^\text{add}_{S, \rho}(h) + \frac{4\sqrt{2}}{\rho m}
\Lambda_1 r_\infty \sqrt{s \log(2N)}
+ 3M\sqrt{\frac{\log\frac{2}{\delta}}{2m}},\\
& R(h) \leq \h R^\text{mult}_{S, \rho}(h) + \frac{4\sqrt{2}M}{\rho m}
\Lambda_1 r_\infty \sqrt{s \log(2N)}
+ 3M\sqrt{\frac{\log\frac{2}{\delta}}{2m}}.
\end{align*}
Similarly, for any $\delta > 0$, with probability at least
$1 - \delta$ over the draw of a sample $S$ of size $m$, the following
holds for all $h \in \cH_2$,
\begin{align*}
& R(h) \leq \h R^\text{add}_{S, \rho}(h) + \frac{4\sqrt{2}}{\rho m}
\Lambda_2 r_2
\sqrt{\textstyle \sum_{i = 1}^m \sum_{f \in F_i} \sum_{y \in \cY_f} |F_i|}
+ 3M\sqrt{\frac{\log\frac{2}{\delta}}{2m}},\\
& R(h) \leq \h R^\text{mult}_{S, \rho}(h) + \frac{4\sqrt{2}M}{\rho m}
\Lambda_2 r_2
\sqrt{\textstyle \sum_{i = 1}^m \sum_{f \in F_i} \sum_{y \in \cY_f} |F_i|}
+ 3M\sqrt{\frac{\log\frac{2}{\delta}}{2m}}.
\end{align*}
\end{corollary}

\subsection{Learning guarantees for multi-class classification with linear hypotheses}

The following result is a direct consequence of
Corollary~\ref{cor:guarantee_linear} and the observation that for
multi-class classification $|F_i| = 1$ and
$d_i = \max_{f \in F_i} |\cY_f| = c$.  Note that our multi-class
learning guarantees hold for arbitrary bounded losses\ignore{ such as
  precision and recall}. To the best of our knowledge this is a novel
result in this setting. In particular, these guarantees apply to the
special case of the standard multi-class zero-one loss
$L(y, y') = 1_{\{y \neq y'\}}$ which is bounded by $M = 1$.\\

\begin{corollary}
\label{cor:guarantee_linear_multiclass}
Fix $\rho > 0$.  For any $\delta > 0$, with probability at least
$1 - \delta$ over the draw of a sample $S$ of size $m$, the following
holds for all $h \in \cH_1$,
\begin{align*}
& R(h) \leq \h R^\text{add}_{S, \rho}(h) +
\frac{4\sqrt{2} \Lambda_1 r_\infty}{\rho }  \sqrt{\frac{c \log(2N)}{m}}
+ 3M\sqrt{\frac{\log\frac{2}{\delta}}{2m}},\\
& R(h) \leq \h R^\text{mult}_{S, \rho}(h) +
\frac{4\sqrt{2}\Lambda_1 r_\infty}{\rho} \sqrt{\frac{c \log(2N)}{m}}
+ 3M\sqrt{\frac{\log\frac{2}{\delta}}{2m}}.
\end{align*}
Similarly, for any $\delta > 0$, with probability at least
$1 - \delta$ over the draw of a sample $S$ of size $m$, the following
holds for all $h \in \cH_2$,
\begin{align*}
& R(h) \leq \h R^\text{add}_{S, \rho}(h) + 
\frac{4\sqrt{2} \Lambda_2 r_2}{\rho }  \sqrt{\frac{c}{m}}
+ 3M\sqrt{\frac{\log\frac{2}{\delta}}{2m}},\\
& R(h) \leq \h R^\text{mult}_{S, \rho}(h) + 
\frac{4\sqrt{2} \Lambda_2 r_2}{\rho }  \sqrt{\frac{c}{m}}
+ 3M\sqrt{\frac{\log\frac{2}{\delta}}{2m}}.
\end{align*}
\end{corollary}

Consider the following set of linear hypothesis:
\begin{align*}
\cH_{2,1} = \set{x \mapsto \bw \cdot \bPsi(x,y) \colon \|\bw\|_{2,1} \leq \Lambda_{2,1}, y \in [c]},
\end{align*}
where $\bPsi(x,y) = (0, \ldots 0, \bPsi_y(x), 0, \ldots, 0)^T \in \Rset^{N_1 \times \ldots, N_c}$ and $\bw = (\bw_1, \ldots, \bw_c)$ with $\|\bw\|_{2,1} = \sum_{y=1}^c \|\bw_y\|_2$. 
In this case,   $\bw \cdot \bPsi(x,y) =  \bw_y \cdot \bPsi_y(x)$. The standard scenario in multi-class classification is when
$\bPsi_y(x) = \bPsi(x)$ is the same for all $y$.\\

\begin{corollary}
\label{cor:guarantee_linear_multiclass_norm12}
Fix $\rho > 0$.  For any $\delta > 0$, with probability at least
$1 - \delta$ over the draw of a sample $S$ of size $m$, the following
holds for all $h \in \cH_{2,1}$,
\begin{align*}
& R(h) \leq \h R^\text{add}_{S, \rho}(h) +
\frac{16 \Lambda_{2,1} r_{2,\infty} (\log(c))^{1/4}}{\rho \sqrt{m} }  
+ 3M\sqrt{\frac{\log\frac{2}{\delta}}{2m}},\\
& R(h) \leq \h R^\text{mult}_{S, \rho}(h) +
\frac{16 \Lambda_{2,1} r_{2,\infty} (\log(c))^{1/4}}{\rho \sqrt{m} }  
+ 3M\sqrt{\frac{\log\frac{2}{\delta}}{2m}},
\end{align*}
where $r_{2,\infty} = \max_{i,y} \| \Psi_y(x_i)\|_2$.
\end{corollary}

\begin{proof}
  By definition of the dual norm and
  $\cH_{2,1}$, the
  following holds:
\begin{align*}
m \h \Rad^G_S(\cH_{2,1})
& = 
\E_\be \bigg[\sup_{\|\bw\|_{2,1} \leq \Lambda} \bw \cdot   
\sum_{i = 1}^m \sum_{y \in [c]} 
\e_{i, y}  \bPsi(x_i, y) \bigg] \\
& =
\Lambda \E_\be \bigg[\bigg \|
\sum_{i = 1}^m \sum_{y \in [c]}
\e_{i, y}  \bPsi(x_i, y)
\bigg\|_{2,\infty} \bigg] \\
& =
\Lambda \E_\be \left[ \max_{y}  
  \bigg\|\sum_{i = 1}^m
\e_{i, y} \Psi_y(x_i)\bigg\|_2  \right]\\
& \leq
\Lambda \Bigg(\E_\be \left[ \max_{y}  
  \bigg\|\sum_{i = 1}^m
\e_{i, y} \Psi_y(x_i)\bigg\|_2^2  \right]\Bigg)^{1/2}\\
&=
\Lambda \Bigg(\E_\be \bigg[ \max_{y}  
  \sum_{i = 1}^m \bigg\|\Psi_y(x_i)\bigg\|_2^2
 + \sum_{i \neq j}  \e_{i, y} \e_{j, y} \Psi_y(x_i) \cdot \Psi_y(x_j)  \bigg]\Bigg)^{1/2}\\
& \leq
\Lambda \Bigg( \max_{y}  
  \sum_{i = 1}^m \bigg\|\Psi_y(x_i)\bigg\|_2^2
 + \E_\be \bigg[ \max_{y} \sum_{i \neq j}  \e_{i, y} \e_{j, y} \Psi_y(x_i) \cdot \Psi_y(x_j)  \bigg]\Bigg)^{1/2}
\end{align*}
By Lemma~\ref{lemma:MaxConcentration} (or Massart's lemma),
the following bound holds:
\begin{align*}
\E_\be \bigg[ \max_{y} \sum_{i \neq j}  \e_{i, y} \e_{j, y} \Psi_y(x_i) \cdot \Psi_y(x_j)  \bigg] \leq m r_{2, \infty} \sqrt{\log(c)}. 
\end{align*}
Since, $\max_{y} \sum_{i = 1}^m \|\Psi_y(x_i)\|_2^2 \leq m r^2_{2, \infty} $,
we obtain that the following result holds:
\begin{align*}
\h \Rad^G_S(\cH_{2,1}) \leq \frac{\sqrt{2} \Lambda r_{2, \infty} (\log(c))^{1/4}}{\sqrt{m}},
\end{align*}
and applying Theorem~\ref{th:bound} completes the proof.
\end{proof}

\subsection{VRM structured prediction learning bounds}
\label{app:VRM_bounds}

Here, we give the proof of our structured prediction learning
guarantees in the setting of Voted Risk Minimization. We will
use the following lemma.\\

\begin{lemma}
\label{lemma:F}
The function $\F$ is sub-additive: $\F(x + y) \leq \F(x) + \F(y)$, for
all $x, y \in \Rset$.
\end{lemma}
\begin{proof} 
  By the sub-additivity of the maximum function, for any
  $x, y \in \Rset$, the following upper bound holds for $\F(x + y)$:
\begin{align*}
\F(x + y) = \min(M, \max(0, x + y)) 
& \leq \min(M, \max(0, x) + \max(0, y)) \\
& \leq \min(M, \max(0, x)) + \min(M, \max(0, y)) \\
& = \F(x) + \F(y),
\end{align*}
which completes the proof.
\end{proof}
For the following proof, for any $\tau \geq 0$, the margin losses
$R^\text{add}_{\rho, \tau}(h)$ and $R^\text{mult}_{\rho, \tau}(h)$ are
defined as the population counterparts of the empirical losses define
by \eqref{eq:add_tau} and \eqref{eq:mult_tau}.\\

\begin{reptheorem}{th:main}
  Fix $\rho > 0$. For any $\delta > 0$, with probability at least
  $1 - \delta$ over the draw of a sample $S$ of size $m$, each of
  the following inequalities holds for all $f \in \cF$:
\begin{align*}
 &R(f) - \h R^\text{add}_{S, \rho, 1}(f)
\leq \frac{4\sqrt{2}}{\rho}
\sum_{t= 1}^T \alpha_t \Rad^G_m(H_{k_t}) + C(\rho, M, c, m, p), \\
 &R(f) - \h R^\text{mult}_{S, \rho, 1}(f)
\leq \frac{4\sqrt{2}M}{\rho}
\sum_{t= 1}^T \alpha_t \Rad^G_m(H_{k_t}) + C(\rho, M, c, m, p).
\end{align*}
where
\begin{align*}
C(\rho, M, c, m, p) = \frac{2M}{\rho} \sqrt{\frac{\log p}{m}}
+ 3 M \sqrt{\Big\lceil \tfrac{4}{\rho^2} \log \big(\tfrac{c^2 \rho^2 m}{4 \log p}\big)
\Big\rceil \frac{\log p}{m} + \frac{\log \frac{2}{\delta}}{2m}}.
\end{align*}
\ignore{Thus, $R(f) \leq \h R_{S, \rho}(f) + 8 \frac{M}{\rho} \sqrt{\frac{ \pi \log(2pc \ov N)}{m}} \sum_{k = 1}^p
\|\bw^k\|_k r_k + O\left(\sqrt{\dfrac{\log
    p}{\rho^2 m} \log \Big[ \frac{\rho^2 M^2c^2 m}{4 \log p} \Big] }\right)$.}
\end{reptheorem}

\begin{proof}
  The proof makes use of Theorem~\ref{th:bound} and the proof
  techniques of \citeapp{KuznetsovMohriSyed2014}[Theorem~1] but requires
  a finer analysis both because of the general loss functions used
  here and because of the more complex structure of the hypothesis
  set.

For a fixed $\bh = (h_1, \ldots, h_T)$, any $\Alpha$ in the probability
simplex $\Delta$ defines a distribution over
$\set{h_1, \ldots, h_T}$. Sampling from
$\set{h_1, \ldots, h_T}$ according to $\Alpha$ and averaging leads to
functions $g$ of the form $g = \frac{1}{n} \sum_{i = 1}^T n_t h_t$ for
some $\n = (n_1, \ldots, n_T) \in \Nset^T$, with $\sum_{t = 1}^T n_t = n$, and
$h_t \in \cH_{k_t}$.

For any $\N = (N_1, \ldots, N_p)$ with $| \N | = 
n$, we consider the family of functions
\begin{equation*}
G_{\cF, \N} = \bigg\{\frac{1}{n} \sum_{k = 1}^p \sum_{j = 1}^{N_k}
  h_{k,j} \mid \forall (k, j) \in [p] \times [N_k], h_{k,j} \in
  H_k \bigg\},
\end{equation*}
and the union of all such families $G_{\cF, n} = \bigcup_{| \N | = n}
G_{\cF, \N}$.  Fix $\rho > 0$. For a fixed $\N$, the empirical factor graph
Rademacher complexity of $G_{\cF, \N}$ can be bounded as follows for any
$m \geq 1$:
\begin{align*}
\h \Rad^G_S(G_{\cF, \N}) \leq \frac{1}{n} \sum_{k = 1}^p
N_k \, \h \Rad^G_S(H_k),
\end{align*}
which also implies the result for the true factor graph Rademacher complexities.

Thus, by Theorem~\ref{th:bound}, the following learning bound holds: for any
$\delta > 0$, with probability at least $1 - \delta$, for all
$g \in G_{\cF, \N}$,
\begin{equation*}
R^\text{add}_{\rho, \frac{1}{2}}(g) -
\h R^\text{add}_{S, \rho, \frac{1}{2}}(g)
\leq \frac{1}{n} \frac{4\sqrt{2}}{\rho} 
\sum_{k = 1}^p N_k \,
\Rad^G_m(H_k) + M\sqrt{\frac{\log \frac{1}{\delta}}{2m}}.
\end{equation*}
Since there are at most $p^n$ possible $p$-tuples $\N$
with $| \N | = n$,\footnote{
The number $S(p, n)$ of $p$-tuples $\N$
with $| \N |=n$ is known to be precisely $\tbinom{p + n - 1}{p - 1}$.\ignore{ If $n \geq p$,
$S(p,n) = \sum_{i=1}^n S(p-1,n-i) \leq n S(p-1, n-1)$. Iterating this
inequality we obtain that $S(n,p) \leq n(n-1)(n-2) \cdots (n-p+1) \leq p^n$.
If $n < p$ the same argument shows that $S(n,p) \leq n^{n-1}p \leq p^n$. }}
by the union bound, for any $\delta > 0$, with probability at
least $1 - \delta$, for all $g \in G_{\cF, n}$, we can write
\begin{equation*}
R^\text{add}_{\rho, \frac{1}{2}}(g) -
 \h R^\text{add}_{S, \rho, \frac{1}{2}}(g)
\leq  \frac{1}{n} \frac{4 \sqrt{2}}{\rho} \sum_{k = 1}^p N_k \,
\Rad^G_m(H_k) + M\sqrt{\frac{\log \frac{p^n}{\delta}}{2m}}.
\end{equation*}
Thus, with probability at least $1 - \delta$, for all functions $g =
\frac{1}{n} \sum_{i = 1}^T n_t h_t$ with $h_t \in \cH_{k_t}$, the following
inequality holds
\begin{equation*}
R^\text{add}_{\rho, \frac{1}{2}}(g) -
 \h R^\text{add}_{S, \rho, \frac{1}{2}}(g)
\leq  \frac{1}{n} \frac{4\sqrt{2 }}{\rho} \sum_{k = 1}^p \sum_{t : k_t = k} n_t \,
\Rad^G_m(H_{k_t}) + M\sqrt{\frac{\log \frac{p^n}{\delta}}{2m}}.
\end{equation*}
Taking the expectation with respect to $\Alpha$ and using $\E_\Alpha[n_t/n] =
\alpha_t$, we obtain that for any $\delta > 0$, with probability at
least $1 - \delta$, for all $g$, we can write
\begin{equation*}
\E_\Alpha[R^\text{add}_{\rho, \frac{1}{2}}(g) -
 \h R^\text{add}_{S, \rho, \frac{1}{2}}(g)]
\leq \frac{4\sqrt{2}}{\rho}
\sum_{t= 1}^T \alpha_t \Rad^G_m(H_{k_t}) +
M\sqrt{\frac{\log \frac{p^n}{\delta}}{2m}}.
\end{equation*}
Fix $n \geq 1$. Then, for any $\delta_n > 0$, with probability at
least $1 - \delta_n$,
\begin{equation*}
\E_\Alpha[R^\text{add}_{\rho, \frac{1}{2}}(g) -
 \h R^\text{add}_{S, \rho, \frac{1}{2}}(g)]
\leq \frac{4 \sqrt{2}}{\rho}
\sum_{t = 1}^T \alpha_t \Rad^G_m(H_{k_t}) +
M\sqrt{\frac{\log \frac{p^n}{\delta_n}}{2m}}.
\end{equation*}
Choose $\delta_n = \frac{\delta}{2 p^{n - 1}}$ for some $\delta > 0$,
then for $p \geq 2$, $\sum_{n \geq 1} \delta_n = \frac{\delta}{2 (1 -
  1/p)} \leq \delta$. Thus, for any $\delta > 0$ and any $n \geq 1$, with
probability at least $1 - \delta$, the following holds for all $g$:
\begin{equation}
\label{eq:gbound_multi}
\E_{\Alpha}[R^\text{add}_{\rho, \frac{1}{2}}(g) -
 \h R^\text{add}_{S, \rho, \frac{1}{2}}(g)]
\leq \frac{4\sqrt{2}}{\rho}
\sum_{t= 1}^T \alpha_t \Rad^G_m(H_{k_t}) +
M\sqrt{\frac{\log \frac{2 p^{2n - 1}}{\delta}}{2m}}.
\end{equation}
Now, for any $f = \sum_{t = 1}^T \alpha_t h_t \in \cF$ and any
$g = \frac{1}{n} \sum_{i = 1}^T n_t h_t$, using \eqref{eq:lossrho}, we
can upper bound $R(f)$, the generalization error of $f$, as follows:
\begin{align}
\label{eq:vrm_aux_bound_1}
R(f) 
& = \E \Big[ L(\sf f(x), y) 1_{\rho_f(x, y) \leq 0} \Big]\\
& \leq \E \Big[ L(\sf f(x), y) 1_{\rho_f(x, y) - (g(x, y) - g(x, y_f))
 < -\rho/2} \Big]
 + \E \Big[ L(\sf f(x), y) 1_{g(x, y) - g(x, y_f) \leq \rho/2} \Big] \nonumber \\
& \leq M  \Pr \Big[ \rho_f(x, y) - (g(x, y) - g(x, y_f)) <
  -\rho/2 \Big]
 + \E \Big[ L(\sf f(x), y) 1_{g(x, y) - g(x, y_f) \leq \rho/2} \Big], \nonumber
\end{align}
where for any function $\varphi\colon \cX \times \cY \to [0, 1]$, we define
$y_\varphi$ as follows:
$y_\varphi = \argmax_{y' \neq y} \varphi(x, y)$.  Using the
same arguments as in the proof of Lemma~\ref{lemma:surrogate}, one can
show that
\begin{align*}
\E \Big[ L(\sf f(x), y)) 1_{g(x, y) - g(x, y_f) < \rho/2} \Big] \leq
R^\text{add}_{\rho, \frac{1}{2}}(g).
\end{align*}

We now give a lower-bound on $\h R^\text{add}_{S, \rho, 1}(f)$ in
terms of $R^\text{add}_{S, \rho, \frac{1}{2}}(g)$. To do so, we start
with the expression of $\h R^\text{add}_{S, \rho, \frac{1}{2}}(g)$:
\begin{align*}
\h R^\text{add}_{S, \rho, \frac{1}{2}}(g) = \E_{(x, y) \sim S}
\Big[ \F \big( \max_{y' \neq y} \loss(y', y) + \tfrac{1}{2} \! - \! \tfrac{1}{\rho}
[g(x, y) \! - \! g(x, y')] \big) \Big]
\end{align*}
By the sub-additivity of $\max$, we can write
\begin{align*}
& \max_{y' \neq y} \loss(y', y) + \tfrac{1}{2} \! - \! \tfrac{1}{\rho}
[g(x, y) \! - \! g(x, y')] \\
& \leq \max_{y' \neq y}
\Bigg\{L(y, y') + 1 - \frac{f(x, y) - f(x, y')}{\rho}\Bigg\} \\
& \quad + \max_{y' \neq y}
\Bigg\{-\frac{1}{2} + \frac{f(x, y) - f(x, y')}{\rho} -
\frac{g(x, y) - g(x, y')}{\rho}\Bigg\}
= X + Y,
\end{align*}
where $X$ and $Y$ are defined by
\begin{align*}
& X = \max_{y' \neq y}
\Bigg(L(y, y') + 1 - \frac{f(x, y) - f(x, y')}{\rho}\Bigg),\\
& Y = -\frac{1}{2} + \max_{y' \neq y}
\Bigg(\frac{f(x, y) - f(x, y')}{\rho} -
\frac{g(x, y) - g(x, y')}{\rho}\Bigg).
\end{align*}
In view of that, since $\F$ is non-decreasing and sub-additive
(Lemma~\ref{lemma:F}), we can write
\begin{align}
\label{eq:vrm_aux_bound_2}
\h R^\text{add}_{S, \rho, \frac{1}{2}}(g) 
& \leq \E_{(x, y) \sim S}[\F(X + Y)]\\
& \leq \E_{(x, y) \sim S}[\F(X) + \F(Y)]
= \E_{(x, y) \sim S}[\F(X)] + \E_{(x, y) \sim S}[\F(Y)]\nonumber \\
& = \h R^\text{add}_{S, \rho, 1}(f) + \E_{(x, y) \sim S}[\F(Y)]\nonumber \\
& \leq \h R^\text{add}_{S, \rho, 1}(f) + M \E_{(x, y) \sim S}[1_{Y
  > 0}]\nonumber \\
& = \h R^\text{add}_{S, \rho, 1}(f) + M \Pr_{(x, y) \sim S}\Big[ \max_{y' \neq y} \big\{
f(x, y) - g(x, y) + (g(x, y') - f(x, y')) \big\} > \rho/2 \Big]. \nonumber 
\end{align}
Combining \eqref{eq:vrm_aux_bound_1} and \eqref{eq:vrm_aux_bound_2}
shows that $R(f) - \h R^\text{add}_{S, \rho, 1}(f)$ is bounded by
\begin{align*}
R^\text{add}_{\rho, \frac{1}{2}}(g) -
 \h R^\text{add}_{S, \rho, \frac{1}{2}}(g)
&+ M  \Pr \Big[ \rho_f(x, y) - (g(x, y) - g(x, y_f)) < -\rho/2 \Big] \\
&+ M \Pr_{(x, y) \sim S}\Big[ \max_{y' \neq y} \{
f(x, y) - g(x, y) + (g(x, y') - f(x, y'))\} > \rho/2 \Big].
\end{align*}
Taking the expectation with respect to $\Alpha$ shows that
$R(f) - \h R^\text{add}_{S, \rho, 1}(f)$ is bounded by
\begin{align}
\label{eq:difference_bound}
\E_{\Alpha}\Big[ R^\text{add}_{\rho, \frac{1}{2}}(g) -
 \h R^\text{add}_{S, \rho, \frac{1}{2}}(g) \Big]
&+ M  \E_{(x, y) \sim \cD, \Alpha}
 \Big[ 1_{\rho_f(x, y) - (g(x, y) - g(x, y_f)) < -\rho/2} \Big] \nonumber \\
&+ M \E_{(x, y) \sim S, \Alpha} \Big[ 1_{\max_{y' \neq y} \{
f(x, y) - g(x, y) + (g(x, y') - f(x, y'))\} > \rho/2} \Big].
\end{align}
By Hoeffding's bound, the following holds:
\begin{align*}
\E_{\Alpha}
 \Big[ 1_{\rho_f(x, y) - (g(x, y) - g(x, y_f)) < -\rho/2} \Big]
& =
\Pr_{\Alpha}
 \Big[ (f(x, y) - f(x, y_f)) - (g(x, y) - g(x, y_f)) < -\rho/2 \Big] \\
& \leq e^{-n\rho^2/8}.
\end{align*}
Similarly, using the union bound and Hoeffding's bound,
the third expectation term appearing in
\eqref{eq:difference_bound} can be bounded as follows:
\begin{align*}
& \mspace{-40mu} \E_{\Alpha} \Big[ 1_{\max_{y' \neq y} \{
f(x, y) - g(x, y) + (g(x, y') - f(x, y'))\} > \rho/2} \Big] \\
& =
\Pr_{\Alpha} \Big[ \max_{y' \neq y} \{
f(x, y) - g(x, y) + (g(x, y') - f(x, y'))\} > \rho/2 \Big] \\
& \leq \sum_{y' \neq y}
\Pr_{ \Alpha} \Big[
f(x, y) - g(x, y) + (g(x, y') - f(x, y')) > \rho/2 \Big] \\
& \leq (c - 1) e^{-n\rho^2/8}.
\end{align*}
Thus, for any fixed $f$, we can write
\begin{align*}
R(f) - \h R^\text{add}_{S, \rho, 1}(f)
& \leq c M e^{-n\rho^2/8} + \E_{\Alpha}\Big[ R^\text{add}_{\rho, \frac{1}{2}}(g) -
 \h R^\text{add}_{S, \rho, \frac{1}{2}}(g) \Big].
\end{align*}
Therefore, the following quantity upper bounds $\sup_{f } R(f) - \h R^\text{add}_{S, \rho, 1}(f)$:
\begin{align*}
c M e^{-n\rho^2/8} + \sup_{g} \E_{\Alpha}\Big[ R^\text{add}_{\rho, \frac{1}{2}}(g) -
 \h R^\text{add}_{S, \rho, \frac{1}{2}}(g) \Big],
\end{align*}
and, in view of \eqref{eq:gbound_multi}, for any $\delta > 0$ and any
$n \geq 1$, with probability at least $1 - \delta$, the following
holds for all $f$:
\begin{align*}
 R(f) - \h R^\text{add}_{S, \rho, 1}(f) 
\leq c M e^{-n\rho^2/8} +  \frac{4\sqrt{2}}{\rho}
\sum_{t= 1}^T \alpha_t \Rad^G_m(H_{k_t}) +
M\sqrt{\frac{\log \frac{2 p^{2n - 1}}{\delta}}{2m}}.
\end{align*}
Choosing $n =  \Big\lceil \frac{4}{\rho^2} \log \big(\frac{ c^2 \rho^2 m}{4 \log p}\big)
\Big\rceil$ yields the following inequality:\footnote{To select $n$ we
  consider $f(n) = c e^{-nu} + \sqrt{n v}$, where $u = \rho^2 / 8$
  and $v = \log p / m$. Taking the derivative of $f$, setting it to zero
  and solving for $n$, we obtain $n = -\frac{1}{2u}
  W_{-1}(-\frac{v}{2 c^2 u})$ where $W_{-1}$ is the second
  branch of the Lambert function (inverse of $x \mapsto x e^x$). Using
  the bound $-\log x \leq -W_{-1}(-x) \leq 2 \log x$ leads to the
  following choice of $n$: $n = \big\lceil -\frac{1}{2u}
  \log(\frac{v}{2 c^2 u})\big\rceil$.}
\begin{align*}
 R(f) - \h R^\text{add}_{S, \rho, 1}(f)
\leq \frac{4\sqrt{2}}{\rho}
\sum_{t= 1}^T \alpha_t \Rad^G_m(H_{k_t}) &+
\frac{2M}{\rho}
\sqrt{\frac{\log p}{m}} \\
&+ 3 M \sqrt{\Big\lceil \tfrac{4}{\rho^2} \log \big(\tfrac{c^2 \rho^2 m}{4 \log p}\big)
\Big\rceil \frac{\log p}{m} + \frac{\log \frac{2}{\delta}}{2m}},
\end{align*}
and concludes the proof.
\end{proof}

By applying Theorem~\ref{th:bound_emp_complexity} instead of Theorem~\ref{th:bound} 
and keeping track of the slightly increased exponential concentration terms
in the proof above, we arrive at the following analogue of Theorem~\ref{th:main} 
in terms of empirical complexities:
\begin{theorem}
  \label{th:main_emp_complexity}
  Fix $\rho > 0$. For any $\delta > 0$, with probability at least
  $1 - \delta$ over the draw of a sample $S$ of size $m$, each of
  the following inequalities holds for all $f \in \cF$:
\begin{align*}
 &R(f) - \h R^\text{add}_{S, \rho, 1}(f)
\leq \frac{4\sqrt{2}}{\rho}
\sum_{t= 1}^T \alpha_t \hat \Rad^G_m(H_{k_t}) + C(\rho, M, c, m, p), \\
 &R(f) - \h R^\text{mult}_{S, \rho, 1}(f)
\leq \frac{4\sqrt{2}M}{\rho}
\sum_{t= 1}^T \alpha_t \hat \Rad^G_m(H_{k_t}) + C(\rho, M, c, m, p).
\end{align*}
where
\begin{align*}
C(\rho, M, c, m, p) = \frac{2M}{\rho} \sqrt{\frac{\log p}{m}}
+ 9 M \sqrt{\Big\lceil \tfrac{4}{\rho^2} \log \big(\tfrac{c^2 \rho^2 m}{4 \log p}\big)
\Big\rceil \frac{\log p}{m} + \frac{\log \frac{2}{\delta}}{2m}}.
\end{align*}
\ignore{Thus, $R(f) \leq \h R_{S, \rho}(f) + 8 \frac{M}{\rho} \sqrt{\frac{ \pi \log(2pc \ov N)}{m}} \sum_{k = 1}^p
\|\bw^k\|_k r_k + O\left(\sqrt{\dfrac{\log
    p}{\rho^2 m} \log \Big[ \frac{\rho^2 M^2c^2 m}{4 \log p} \Big] }\right)$.}
\end{theorem}

\subsection{General upper bound on the loss based on convex
  surrogates}
\label{app:convex_surrogates}

Here, we present the proof of a general upper bound on a loss function
in terms of convex surrogates.\\

\begin{replemma}{lemma:surrogate}
For any $u \in \Rset_+$, let $\Phi_u\colon \Rset \to \Rset$ be an
upper bound on $v \mapsto u \I_{v \leq 0}$\ignore{ such that
$u \mapsto \Phi_u(v)$ is increasing for a fixed $v$}. Then, the
following upper bound holds for any $h \in \cH$ and
$(x, y) \in \cX \times \cY$,
\begin{equation}
\loss(\hh(x), y)
\leq \max_{y' \neq y} \Phi_{\loss( y', y)} (h(x, y) - h(x, y')).
\end{equation}
\end{replemma}
\begin{proof}
  If $\hh(x) = y$, then $\loss(\hh(x), y) = 0$ and the result follows.
  Otherwise, $\hh(x) \neq y$ and the following bound holds:
\begin{align*}
\loss(\hh(x), y)
& = \loss(\hh(x), y) \I_{\rho_h(x, y) \leq 0} \\
& \leq \Phi_{\loss(\hh(x), y)}(\rho_h(x, y))\\
& = \Phi_{\loss(\hh(x),y)} (h(x, y) - \max_{ y' \neq y} h(x,  y')) \\
& = \Phi_{\loss(\hh(x),y)} (h(x, y) -  h(x,  \hh(x))) \\
&\leq \max_{y' \neq y}\Phi_{\loss(y',y)} (h(x, y) -  h(x,  y')),
\end{align*}
which concludes the proof.   
\end{proof}

\section{Experiments}
\label{app:experiments}

\subsection{Datasets}
\label{app:datasets}

\begin{table}[t]
\caption{Description of datasets.}
\label{table:data_desc}
\scriptsize
\vskip .1in
\centering
\setlength{\tabcolsep}{0.15 cm}
\begin{tabular}{ l  c  c  c  c  c  c }
  \hline
  Dataset            & Full name               & Sentences & Tokens & Unique tokens & Labels \\ \hline
  {\tt Basque}       & Basque UD Treebank      & 8993   & 121443  & 26679 & 16 \\ 
  {\tt Chinese}      & Chinese Treebank 6.0    & 28295  & 782901  & 47570 & 37 \\ 
  {\tt Dutch}        & UD Dutch Treebank       & 13735  & 200654  & 29123 & 16 \\ 
  {\tt English}      & UD English Web Treebank & 16622  & 254830  & 23016 & 17 \\ 
  {\tt Finnish}      & Finnish UD Treebank     & 13581  & 181018  & 53104 & 12 \\ 
  {\tt Finnish-FTB}  & UD\_Finnish-FTB         & 18792  & 160127  & 46756 & 15 \\ 
  {\tt Hindi}        & UD Hindi Treebank             & 16647  & 351704  & 19232 & 16 \\ 
  {\tt Tamil}        & UD Tamil Treebank             & 600    & 9581    & 3583  & 14 \\ 
  {\tt Turkish}      & METU-Sabanci Turkish Treebank & 5635   & 67803   & 19125 & 32 \\ 
  {\tt Twitter}      & Tweebank                      & 929    & 12318   & 4479  & 25 \\ \hline 
\end{tabular}
\end{table}

This section reports the results of preliminary experiments with
the \VCRF\ algorithm.  The experiments in this section are meant to
serve as a proof of concept of the benefits of VRM-type regularization
as suggested by the theory developed in this paper.  We leave an
extensive experimental study of other aspects of our theory, including
general loss functions, convex surrogates and $p$-norms, to
future work.

For our experiments, we chose the part-of-speech task
(POS) that consists of labeling each word of a sentence with its
correct part-of-speech tag.  We used 10 POS datasets: {\tt Basque},
{\tt Chinese}, {\tt Dutch}, {\tt English}, {\tt Finnish}, {\tt
  Finnish-FTB}, {\tt Hindi}, {\tt Tamil}, {\tt Turkish} and {\tt
  Twitter}.  The detailed description of these datasets is in
Appendix~\ref{app:datasets}.
Our \VCRF\ algorithm can be applied with a variety of different
families of feature functions $H_k$ mapping $\cX \times \cY$ to
$\Rset$. Details concerning features and complexity penalties $r_k$s
are provided in Appendix~\ref{app:features}, while an outline of our
hyperparameter selection and cross-validation procedure is given in
Appendix~\ref{app:tuning}.

The average error and the standard deviation of the errors are
reported in Table~\ref{table:results} for each data set.  Our results
show that \VCRF\ provides a statistically significant improvement over
$L_1$-CRF on every dataset, with the exception of {\tt English} and
{\tt Dutch}.  One-sided paired $t$-test at $5\%$ level was used to
assess the significance of the results.  It should be noted that for
all of the significant results, \VCRF\ outperformed $L_1$-CRF on every
fold.  Furthermore, our results indicate that \VCRF\ tends to produce
models that are sparser than those of $L_1$-CRF. This is highlighted
in Table~\ref{table:ftr_cnt} of Appendix~\ref{app:features}.
As can be seen, \VCRF\ tends to produce
models that are much more sparse due to its heavy penalization on the
large number of higher-order features.
In a separate set of experiments, we have also tested the robustness
of our algorithm to erroneous annotations and noise.  The details and
the results of these experiments are given in
Appendix~\ref{app:results_noise}.

Further details on the datasets and the specific features as well as more experimental
results are provided below. 

Table~\ref{table:data_desc} provides some statistics for each of the datasets that we use.
These datasets span a variety of sizes, in terms of sentence count, token count, and unique
token count. Most are annotated under the Universal Dependencies (UD) annotation system\ignore{ \cite{NivreEtAl2015}},
with the exception of the Chinese (\citeapp{PalmerEtAl2007}), Turkish (\citeapp{OflazerEtAl2003, AtalayEtAl2003}), and Twitter (\citeapp{GimpelEtAl2011, OwoputiEtAl2013}) datasets.

\subsection{Features and complexities}
\label{app:features}

The standard features that are used in POS tagging are usually binary
indicators that signal the occurrence of certain words, tags or other
linguistic constructs such as suffixes, prefixes, punctuation,
capitalization or numbers in a window around a given position in the
sequence.  In our experiments, we use the union of a broad family of
products of such indicator functions.  Let $V$ denote the input
vocabulary over alphabet $\Sigma$. For $x \in V$ and $t \geq 0$, let
$\text{suff}(x,t)$ be the suffix of length $t$ for the word $x$ and
$\text{pref}(x,t)$ the prefix. Then for $k_1, k_2, k_3 \geq 0$, we can
define the following three families of base features:
\begin{align*}
  &H^\text{w}_{k_1}(s) = \Big\{ x \mapsto \I_{x^{s+r}_{s-t+1} = x'} 
  \colon t,r \in \Nset , r+t=k_1, x' \in V^{k_1}\Big\}, \\
  &H^\text{tag}_{k_2}(s) = \set{ y \mapsto \I_{y^{s}_{s-k_2+1} = y'} \colon y' \in \Delta^{k_2}}, \\
 &H^\text{sp}_{k_3}(s) = \Big\{ x \mapsto \I_{\text{suff}(x_s, t) = S} \I_{\text{pref}(x_s, r) = P}
  \colon t,r \in \Nset, t + r = k_3, S \in \Sigma^t, P \in \Sigma^r\Big\}.
\end{align*}
We can then define a family of features $H_{k_1,k_2,k_3}$
that consists of functions of the form
\begin{align*}
\Psi(x, y) = \sum_{s=1}^l \psi(x, y, s),
\end{align*}
where $\psi(x, y, s) = h_1(x)h_2(y)h_3(x)$,
for some 
$h_1 \in \cH^\text{w}_{k_1}(s)$, $h_2 \in \cH^\text{tag}_{k_2}(s)$,
$h_3 \in \cH^\text{sp}_{k_3}(s)$.

As an example, consider the following sentence:
\begin{table}[h]
\begin{center}
\begin{tabular}{c c c c c}
  DET & NN  & VBD & RB          & JJ \\
  The & cat & was & surprisingly & agile \\ 
\end{tabular}
\end{center}
\end{table}

\begin{figure}[t]
\centering
\includegraphics[scale=0.25]{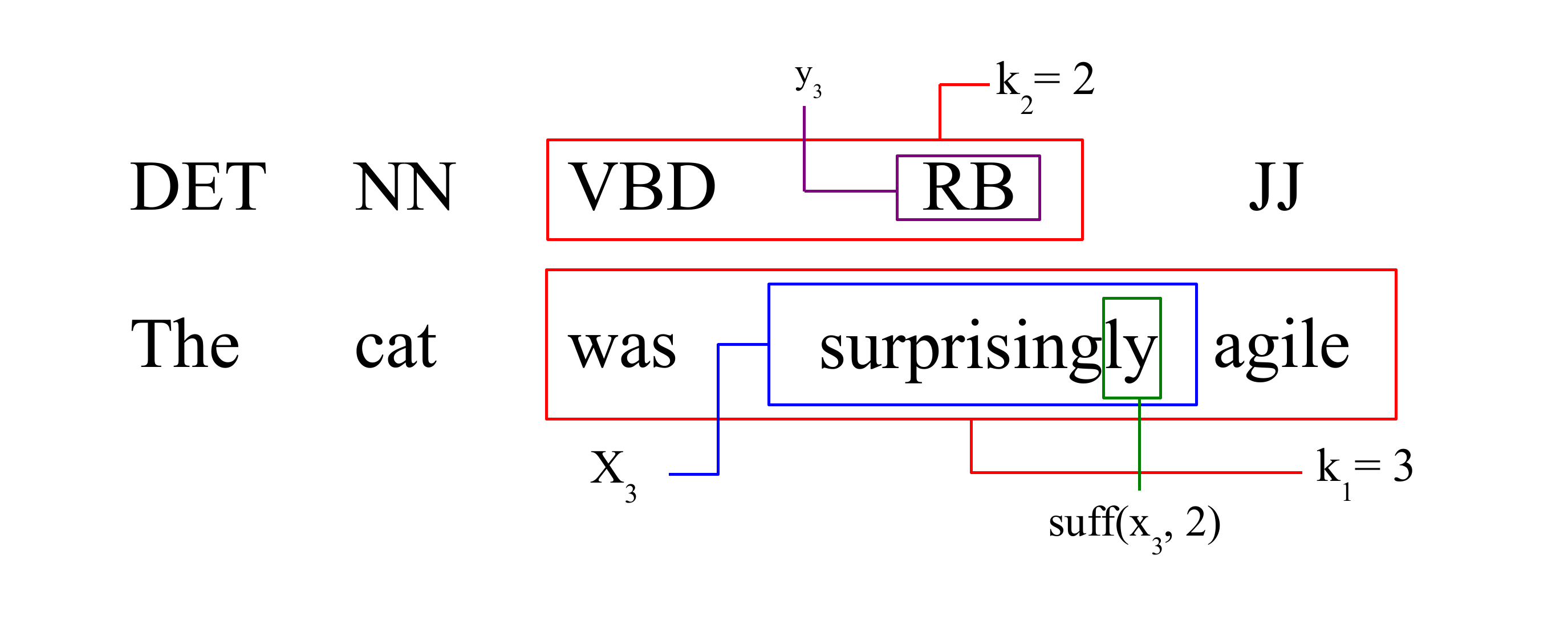}
\vskip -.15in
\caption{Example of features for a POS task.}
\label{fig:pos_example}
\vskip -.15in
\end{figure}

Then, at position $s=3$, the following features 
$h_1 \in \cH_3^\text{w}(3)$, $h_2 \in \cH_2^\text{tag}(3)$, $h_3 \in \cH_1^\text{sp}(3)$ would activate:
\begin{align*}
 &h_1(x) = \I_{x_2 = \text{`was'},\, x_3 = \text{`surprisingly'},\, x_4 = \text{`agile'}}(x) \\
  &h_2(y) = \I_{y_2=\text{'VBD'},\, y_3 = \text{`RB'}}(y) \\
  &h_3(x) = \I_{\text{suff}(x_3, 2) = \text{`ly'}}(x). 
\end{align*}

See Figure~\ref{fig:pos_example} for an illustration.

Now, recall that the \VCRF\ algorithm requires knowledge of 
complexities $r(H_{k_1,k_2,k_3})$.
By definition of the hypothesis set and $r_k$s
\begin{align}
\label{eq:rademacher-bound}
r(H_{k_1,k_2,k_3}) \leq
\sqrt{\frac{2 (k_1 \log |V| + k_2 \log |\Delta| + k_3 \log |\Sigma|}{m}},
\end{align}
which is precisely the complexity penalty used in our experiments.
\ignore{
This can be seen as a consequence of Massart's lemma, which states
that for $A \subset \Rset^n$,
$r = \max_{{\bf x} \in A} \|{\bf x}\|_2$, and $\sigma_i$'s Rademacher
random variables, we have
\begin{align*}
\frac{1}{m} \E_{{\bf \sigma}}\Big[\sup_{{\bf x} \in A} \sum_{i=1}^m \sigma_i x_i\Big] \leq \frac{r \sqrt{2\log |A|}}{m}.
\end{align*}

In our case, we take $A$ to be the image of the sample data under
the function family $H_{k_1, k_2, k_3}$, so that 
$|A| \leq |H_{k_1,k_2,k_3}| \leq |V|^{k_1} |\Delta|^{k_2} |\Sigma|^{l_3}$.
Moreover, since our features are indicator functions, it suffices to set
$r = \sqrt{m}$. Applying Massart's lemma leads
to the desired bound.}

The impact of this added penalization can be seen in Table~\ref{table:ftr_cnt}, where
it is seen that the number of non-zero features for \VCRF~can be
dramatically smaller than the number for $L_1$-regularized CRF.

\subsection{Hyperparameter tuning and cross-validation}
\label{app:tuning}

Recall that the \VCRF\ algorithm admits two hyperparameters
$\lambda$ and $\beta$.  In our experiments, we optimized over
$\lambda, \beta \in \set{1, 0.5, 10^{-1}, \ldots, 10^{-5}, 0}$.  We
compared \VCRF\ against $L_1$-regularized CRF, which is the special
case of \VCRF\ with $\lambda = 0$.  For gradient computation, we used
the procedure in Section~\ref{sec:GradNoloss}, which is agnostic to
the choice of the underlying loss function.  While our algorithms can
be used with very general families of loss functions this choice allows
an easy direct comparison with the CRF algorithm.  We ran each
algorithm for 50 full passes over the entire training set or until
convergence.

\begin{table}[t]
\caption{Experimental results for both \VCRF\ and CRF.
\VCRF\ refers to the conditional random
field objective with both VRM-style regularization and $L_1$ regularization
while CRF refers to the objective with only $L_1$ regularization. 
Boldfaced results are statistically significant at a 5\% confidence level.}
\label{table:results}
\vskip .1in
\scriptsize
\centering
\setlength{\tabcolsep}{0.15 cm}
\begin{tabular}{ l c  c  c  c }
                    & \multicolumn{2}{c } {\textbf{\VCRF\ error (\%)}}             & \multicolumn{2}{c } {\textbf{CRF error(\%)}}  \\ \hline
  Dataset           & Token                       & Sentence                      & Token                       & Sentence \\ \hline
  {\tt Basque}      & {\bf 7.26 $\pm$ 0.13}  & {\bf 57.67 $\pm$ 0.82}   &  \phantom{1}7.68 $\pm$ 0.20  &  59.78 $\pm$ 1.39 \\ 
  {\tt Chinese}     & {\bf 7.38 $\pm$ 0.15}  &  {\bf 67.73 $\pm$ 0.46}   &  \phantom{1}7.67 $\pm$ 0.12  &  68.88 $\pm$ 0.49 \\ 
  {\tt Dutch}       & 5.97 $\pm$ 0.08  & 49.27 $\pm$ 0.71   & \phantom{1}6.01 $\pm$ 0.92  &  49.48 $\pm$ 1.02 \\ 
  {\tt English}     & 5.51 $\pm$ 0.04           & 44.40 $\pm$ 1.30            & \phantom{1}5.51 $\pm$ 0.06            & 44.32 $\pm$ 1.31 \\ 
  {\tt Finnish}     & {\bf 7.48 $\pm$ 0.05}  & {\bf 55.96 $\pm$ 0.64}   &  \phantom{1}7.86 $\pm$ 0.13  &  57.17 $\pm$ 1.36 \\ 
  {\tt Finnish-FTB} & {\bf 9.79 $\pm$ 0.22}  & {\bf 51.23 $\pm$ 1.21}   &  10.55 $\pm$ 0.22 &  52.98 $\pm$ 0.75 \\ 
  {\tt Hindi}       & {\bf 4.84 $\pm$ 0.10}  & {\bf 51.69 $\pm$ 1.07}   &  \phantom{1}4.93 $\pm$ 0.08 & 53.18 $\pm$ 0.75 \\ 
  {\tt Tamil}       & {\bf 19.82 $\pm$ 0.69} & {\bf 89.83 $\pm$ 2.13}  & 22.50 $\pm$ 1.57 & 92.00 $\pm$ 1.54 \\ 
  {\tt Turkish}     & {\bf 11.28 $\pm$ 0.40} & {\bf 59.63 $\pm$ 1.55}   &  11.69 $\pm$ 0.37 &  61.15 $\pm$ 1.01 \\ 
  {\tt Twitter}     & {\bf 17.98 $\pm$ 1.25} & {\bf 75.57 $\pm$ 1.25}   &  19.81 $\pm$ 1.09 &  76.96 $\pm$ 1.37 \\ \hline
\end{tabular}
\end{table} 

In each of the experiments, we used 5-fold cross-validation
for model selection and performance evaluation.
Each dataset was randomly partitioned
into $5$ folds, and each algorithm was run $5$ times, with a
different assignment of folds to the training set, validation
set and test set for each run. For each run
$i \in \{0,\ldots,4\}$, fold $i$ was used for validation, 
fold $i+1 (\text{ mod } 5)$ was used for testing,
and the remaining folds were used
for training. In each run, we selected the parameters that
had the lowest token error on the validation set and then measured
the token and sentence error of those parameters on the test set. The average
error and the standard deviation of the errors 
are reported in Table~\ref{table:results} for each data set.

\subsection{More experiments}
\label{app:results_noise}

In this section, we present our results for a POS tagging task when
noise is artificially injected into the labels. Specifically, for
tokens corresponding to features that commonly appear in the dataset
(at least five times in our experiments), we flip their associated POS
label to some other arbitrary label with 20\% probability.

The results of these experiments are given in 
Table~\ref{table:results_noise20}. They
demonstrate that \VCRF\ outperforms $L_1$-CRF in the majority of cases.
Moreover, these differences can be magnified from the original
scenario, as can be seen on the {\tt English} and {\tt Twitter}
datasets.

\begin{table}[t]
\vskip .1in
\caption{
Average number of features for \VCRF\ and $L_1$-CRF.}
\scriptsize
\centering
\setlength{\tabcolsep}{0.15 cm}
\begin{tabular}{ l  l  l l } \hline
  Dataset           & VCRF & CRF & Ratio \\ \hline
  {\tt Basque}      & 7028 & 94712653 & 0.00007 \\
  {\tt Chinese}     & 219736 & 552918817 & 0.00040\\
  {\tt Dutch}       & 2646231  & 2646231 & 1.00000  \\
  {\tt English}     & 4378177 & 357011992 & 0.01226\\
  {\tt Finnish}     & 32316 & 89333413 & 0.00036 \\
  {\tt Finnish-FTB} & 53337 & 5735210 & 0.00930 \\
  {\tt Hindi}       & 108800 & 448714379 &  0.00024 \\
  {\tt Tamil}       & 1583  & 668545 & 0.00237\\
  {\tt Turkish}     & 498796  & 3314941 & 0.15047 \\
  {\tt Twitter}     & 18371 & 26660216 & 0.000689 \\ \hline
\end{tabular}
\label{table:ftr_cnt}
\end{table}

\begin{table}[t]
\caption{
Experimental results of both VCRF and CRF with 20\% random noise added to the
training set. Labels of tokens are flipped uniformly at random with 
20\% probability. Boldfaced results are statistically significant at a 5\% confidence level.}
\label{table:results_noise20}
\vskip .1in
\scriptsize
\centering
\setlength{\tabcolsep}{0.15 cm}
\begin{tabular}{ l  c  c   c  c }
  
                     & \multicolumn{2}{c } {\textbf{VCRF error (\%)}}           & \multicolumn{2}{c } {\textbf{CRF error(\%)}}  \\ \hline   
  Dataset            & Token                        & Sentence                   & Token               & Sentence \\ \hline
  {\tt Basque}       & {\bf 9.13 $\pm$ 0.18}          & {\bf 67.43 $\pm$ 0.93}         & 9.42 $\pm$ 0.31  & 68.61 $\pm$ 1.08 \\ 
  {\tt Chinese}      & {\bf 96.43 $\pm$ 0.33}      & {\bf 100.00 $\pm$ 0.01}  & 96.81 $\pm$ 0.43  & 100.00 $\pm$ 0.01 \\ 
  {\tt Dutch}        & {\bf 8.16 $\pm$ 0.52}     & {\bf 62.15 $\pm$ 1.77}   & 8.57 $\pm$ 0.30  & 63.55 $\pm$ 0.87 \\ 
  {\tt English}      & {\bf 8.79 $\pm$ 0.23}      & {\bf 61.27 $\pm$ 1.21}   & 9.20 $\pm$ 0.11   & 63.60 $\pm$ 1.18 \\ 
  {\tt Finnish}      & {\bf 9.38 $\pm$ 0.27}      & {\bf 64.96 $\pm$ 0.89}   & 9.62 $\pm$ 0.18   & 65.91 $\pm$ 0.93 \\ 
  {\tt Finnish-FTB}  & {\bf 11.39 $\pm$ 0.29}     & {\bf 72.56 $\pm$ 1.30}   & 11.76 $\pm$ 0.25  & 73.63 $\pm$ 1.19 \\ 
  {\tt Hindi}        & {\bf 6.63 $\pm$ 0.51}            & {\bf 63.84 $\pm$ 2.86}            & 7.85 $\pm$ 0.33  & 71.93 $\pm$ 1.20 \\ 
  {\tt Tamil}        & 20.77 $\pm$ 0.70      & 93.00 $\pm$ 1.35      & 21.36 $\pm$ 0.86  & 93.50 $\pm$ 1.78 \\ 
  {\tt Turkish}      & 14.28 $\pm$ 0.46           & 69.72 $\pm$ 1.51         & 14.31 $\pm$ 0.53  & 69.62 $\pm$ 2.04	 \\ 
  {\tt Twitter}      & 90.92 $\pm$ 1.67      & 100.00 $\pm$ 0.00   & 92.27 $\pm$ 0.71   & 100.00 $\pm$ 0.00 \\ \hline
\end{tabular}
\vskip -.2in
\end{table}

\section{Voted Structured Boosting (\VStructBoost)}
\label{sec:VStructBoost}

In this section, we consider algorithms based on the \StructBoost\ surrogate
loss, where we choose $\Phi_u(v) = u e^{-v}$. Let
$\d \bPsi(x, y, y') = \bPsi(x, y) - \bPsi(x, y')$. This then
leads to the following optimization problem:
\begin{align}
\label{eq:strboost-opt1}
  \min_{\bw} \frac{1}{m} \sum_{i = 1}^m \max_{y \neq y_i} \loss(y, y_i)
e^{-\bw \cdot \d \bPsi(x_i, y_i, y) }
 + \sum_{k = 1}^p (\lambda r_k + \beta) \| \bw_k \|_1. 
\end{align}
One disadvantage of this formulation is that the first term of the
objective is not differentiable. Upper bounding the maximum by a
sum leads to the following optimization problem:
\begin{align}
\label{eq:strboost-opt2} 
\min_{\bw} \frac{1}{m} \sum_{i = 1}^m \sum_{y \neq y_i} \loss(y, y_i)
e^{-\bw \cdot \d \bPsi(x_i, y_i, y) }
 + \sum_{k = 1}^p (\lambda r_k + \beta) \| \bw_k \|_1. 
\end{align}
We refer to the learning algorithm based on the optimization
problem~\eqref{eq:strboost-opt2} as \VStructBoost. To the best of our
knowledge, the formulations \eqref{eq:strboost-opt1} and
\eqref{eq:strboost-opt2} are new, even with the standard $L_1$- or
$L_2$-regularization.

\section{Optimization solutions}
\label{app:opt}

Here, we show how the optimization problems in \eqref{eq:crf-opt1} and
\eqref{eq:strboost-opt2} can be solved efficiently when the feature
vectors admit a particular factor graph decomposition that we refer
to as Markov property. 
\ignore{
Our algorithms work with several broad families of loss
functions including Markovian losses (which include the Hamming loss),
rational losses (which include the $n$-gram loss that has been used to
approximate the BLEU score), and tropical losses (which include the
edit-distance).
}

\subsection{Markovian features}
\label{sec:features}

\ignore{
Learning and inference for structured prediction algorithms is in
general intractable with arbitrary feature mappings $\bPsi$. In
practice, however, the feature mappings used often admit some
favorable properties that, combined with the structure of $\cY$, can
be exploited to derive efficient algorithms.}

We will consider in what follows the common case where $\cY$ is a set
of sequences of length $l$ over a finite alphabet $\Delta$ of size
$r$. Other structured problems can be treated in similar ways. We will
denote by $\ve$ the empty string and for any sequence
$y = (y_1, \ldots, y_l) \in \cY$, we will denote by
$y_{s}^{s'} = (y_s, \ldots, y_{s'})$ the substring of $y$ starting
at index $s$ and ending at $s'$. For convenience, for $s \leq 0$, we
define $y_s$ by $y_s = \ve$.

One common assumption that we shall adopt here is that the feature
vector $\bPsi$ admits a \emph{Markovian property of order $q$}. By this,
we mean that it can be decomposed as follows for
any $(x, y) \in \cX \times \cY$:
\begin{equation}
\bPsi(x, y) = \sum_{s = 1}^l \bpsi(x, y_{s - q + 1}^s, s).
\end{equation}
for some position-dependent feature vector function $\bpsi$ defined
over $\cX \times \Delta^q \times [l]$.  This also suggests a natural
decomposition of the family of feature vectors
$\bPsi = (\bPsi_1, \ldots, \bPsi_p)$ for the application of VRM
principle where $\bPsi_k$ is a Markovian feature vector of order $k$.
Thus, $\cF_k$ then consists of the family of Markovian feature
functions of order $k$.  We note that we can write
$\bPsi = \sum_{k = 1}^p \tilde \bPsi_k$ with
$\tilde \bPsi_k = (0, \ldots, \bPsi_k, \ldots, 0)$. In the following,
abusing the notation, we will simply write $\bPsi_k$ instead of
$\tilde \bPsi_k$. Thus, for any $x \in \cX$ and
$y \in \cY$,\footnote{Our results can be straightforwardly
  generalized to more complex decompositions of the form
  $\bPsi(\bx, y) = \sum_{q = 1}^Q \sum_{k = 1}^p \bPsi_{q, k}(x,
  y)$.}
\begin{equation}
\bPsi(\bx, y) = \sum_{k = 1}^p \bPsi_k(x, y). 
\end{equation}
For any $k \in [1, p]$, let $\bpsi_k$ denote the position-dependent
feature vector function corresponding to $\bPsi_k$. Also, for any
$x \in \cX$ and $y \in \Delta^l$, define $\tl \bpsi$ by
$\tl \bpsi(x, y_{s - p + 1}^s, s) = \sum_{k = 1}^p \bpsi_k(x, y_{s
  - k + 1}^s, s)$. Observe then that we can write
\begin{align}
\bPsi(x, y)
= \sum_{k = 1}^p
  \bPsi_k(x, y)
& = \sum_{k = 1}^p \sum_{s = 1}^l \nonumber
  \bpsi_k(x, y_{s - k + 1}^s, s)\\\nonumber
& = \sum_{s = 1}^l \sum_{k = 1}^p 
  \bpsi_k(x, y_{s - k + 1}^s, s)\\
\label{eqn:features}
& = \sum_{s = 1}^l \tl \bpsi(x_i, y_{s - p + 1}^s, s).
\end{align}
In Sections~\ref{sec:Grad_vcrf} and \ref{sec:Grad_vsb}, we describe
algorithms for efficiently computing the gradient by leveraging the
underlying graph structure of the problem.

\subsection{Efficient gradient computation for \VCRF}
\label{sec:Grad_vcrf}

In this section, we show how Gradient Descent (GD) and Stochastic
Gradient Descent (SGD) can be used to solve the optimization problem
of \VCRF. To do so, we will show how the
subgradient of the contribution to the objective function of a given point
$x_i$ can be computed efficiently. Since the computation of the
subgradient of the regularization term presents no difficulty, it
suffices to show that the gradient of $F_i$, the contribution of point
$x_i$ to the empirical loss term for an arbitrary $i \in [m]$, can be
computed efficiently. In the special case of the Hamming loss or
when loss is omitted from the objective altogether, this coincides
with the standard CRF training procedure. We extend this to more general
families of loss function.

Fix $i \in [m]$. For the \VCRF\ objective, $F_i$ can be rewritten as follows:
\begin{equation*}
F_i(\bw) 
= \frac{1}{m} \log\bigg(\sum_{y \in \cY}
             e^{\loss(y, y_i) - \bw \cdot \d \bPsi(x_i, y_i, y) }
             \bigg)
= \frac{1}{m} \log\bigg(\sum_{y \in \cY}
             e^{\loss(y, y_i) + \bw \cdot \bPsi(x_i, y) }
             \bigg) 
- \frac{\bw \cdot \bPsi(x_i, y_i)}{m}.
\end{equation*}
The following lemma gives the expression of the gradient of $F_i$
and helps identify the key computationally challenging terms $\qq_\bw$. 
\begin{lemma}
\label{lemma:vcrf_grad}
The gradient of $F_i$ at any $\bw$ can be expressed as follows:
\begin{equation*}
\nabla F_i(\bw) 
= \frac{1}{m} \sum_{s = 1}^l \sum_{\bz \in \Delta^p} \Bigg[ \sum_{y\colon y_{s - p +
  1}^s = \bz} \qq_\bw (y) \Bigg] \tl \bpsi(x_i, \bz, s) 
  - \frac{\bPsi(x_i, y_i)}{m},
\end{equation*}
where, for all $y \in \cY$,
\begin{align*}
  &\qq_\bw (y) = \frac{e^{\loss(y, y_i) + \bw \cdot \bPsi(x_i,
    y)}}{Z_\bw},  \\
    &Z_\bw = \sum_{y
    \in \cY} e^{\loss(y, y_i) + \bw \cdot \bPsi(x_i, y)}.
\end{align*}
\end{lemma}

\begin{proof}
  In view of the expression of $F_i$ given above, the gradient of
  $F_i$ at any $\bw$ is given by
\begin{align*}
\nabla F_i(\bw) 
& = \frac{1}{m} \sum_{y \in \cY} \frac{e^{\loss(y, y_i) + \bw
  \cdot \bPsi(x_i, y)} }{\sum_{\tilde{y} \in \cY}
  e^{\loss(\tilde{y}, y_i) + \bw \cdot \bPsi(x_i, \tilde{y})}}
  \bPsi(x_i, y) - \frac{\bPsi(x_i, y_i)}{m} \\
  & = \frac{1}{m} \E_{y \sim \qq_\bw} [\bPsi(x_i, y)] - \frac{\bPsi(x_i, y_i)}{m}.
\end{align*}
By \eqref{eqn:features}, we can write
\begin{align*}
  &\E_{y \sim \qq_\bw} [\bPsi(x_i, y)]  = \sum_{y \in \Delta^l} \qq_\bw (y) \sum_{s = 1}^l \tl
  \bpsi(x_i, y_{s - p + 1}^s, s) 
 = \sum_{s = 1}^l \sum_{\bz \in \Delta^p} \Bigg[ \sum_{y\colon y_{s - p +
  1}^s = \bz} \qq_\bw (y) \Bigg] \tl \bpsi(x_i, \bz, s),
\end{align*}
which completes the proof.
\end{proof}

The lemma implies that the key computation in the gradient is
\begin{align}
\label{eq:marginal-dist}
\QQ_\bw(\bz, s) 
 = \sum_{y\colon y_{s - p + 1}^s = \bz} \qq_\bw (y)
 = \sum_{y\colon y_{s - p + 1}^s = \bz}
\frac{e^{\loss(y, y_i)} \, \prod_{t = 1}^l 
e^{\bw \cdot \tl \bpsi(x_i, y_{t - p + 1}^t, t)}}{Z_\bw},
\end{align}
for all $s \in [l]$ and $\bz \in \Delta^p$. 
The sum defining these terms is over a number of sequences $y$ that
is exponential in $|\Delta|$. However, we will show in the following
sections how to efficiently compute $\QQ_\bw(\bz, s)$ for any
$s \in [l]$ and $\bz \in \Delta^p$ in several important cases: (0) in
the absence of a loss; (1) when $\loss$ is Markovian; (2) when $\loss$
is a \emph{rational loss}; and (3) when $\loss$ is the edit-distance
or any other \emph{tropical loss}.

\subsubsection{Gradient computation in the absence of a loss}
\label{sec:GradNoloss}

In that case, it suffices to show how to compute
$Z'_\bw = \sum_{y \in \cY} e^{\bw \cdot \bPsi(x_i, y)}$ and the
following term, ignoring the loss factors:
\begin{equation}
\QQ'_\bw(\bz, s) = \sum_{y\colon y_{s - p + 1}^s = \bz}
    \prod_{t = 1}^l e^{\bw \cdot \tl \bpsi(x_i, y_{t - p + 1}^t, t)},
\end{equation}
for all $s \in [l]$ and $\bz \in \Delta^p$. 
We will show that $\QQ'_\bw(\bz, s)$ coincides with the flow through an
edge of a weighted graph we will define, which leads to an efficient
computation. We will use for any $y \in \Delta^l$, the convention
$y_s = \ve$ if $s \leq 0$.  Now, let $\scrA$ be the weighted finite
automaton (WFA) with the following set of states:
\begin{equation*}
Q_\scrA = \set[\Big]{(y_{t - p + 1}^t, t)\colon y \in \Delta^l, t = 0, \ldots, l},
\end{equation*}
with $I_\scrA = (\ve, 0)$ its single initial state,
$F_\scrA = \set{(y_{l - p + 1}^l, l)\colon y \in \Delta^l}$
its set of final states, and a transition from state
$(y_{t - p + 1}^{t - 1}, t - 1)$ to state
$(y_{t - p + 2}^{t - 1} \, b, t)$ with label $b$ and weight
$\omega(y_{t - p + 1}^{t - 1} \,
  b, t) = e^{\bw \cdot \tl \bpsi(x_i, y_{t - p + 1}^{t - 1}
  b, t)}$, that is the following set of transitions:
\begin{align*}
E_\scrA 
= \Big\{ & \Big( (y_{t - p + 1}^{t - 1}, t - 1), b, \omega(y_{t - p + 1}^{t - 1} \,
  b, t), (
  y_{t - p + 2}^{t - 1} \, b, t) \Big) \colon
 y \in \Delta^l, b \in \Delta, t \in [l] \Big\}.
\end{align*}
Figure~\ref{fig:wfa} illustrates this construction in the case
$p = 2$. The WFA $\scrA$ is deterministic by construction. The weight
of a path in $\scrA$ is obtained by multiplying the weights of its
constituent transitions. In view of that, 
$\QQ'_\bw(\bz, s)$ can be seen as the sum of the weights of all paths in
$\scrA$ going through the transition from state
$(\bz_1^{p - 1}, s - 1)$ to $(\bz_2^p, s)$ with label $z_p$.

\begin{figure}[t] 
\centering
\includegraphics[scale=.33]{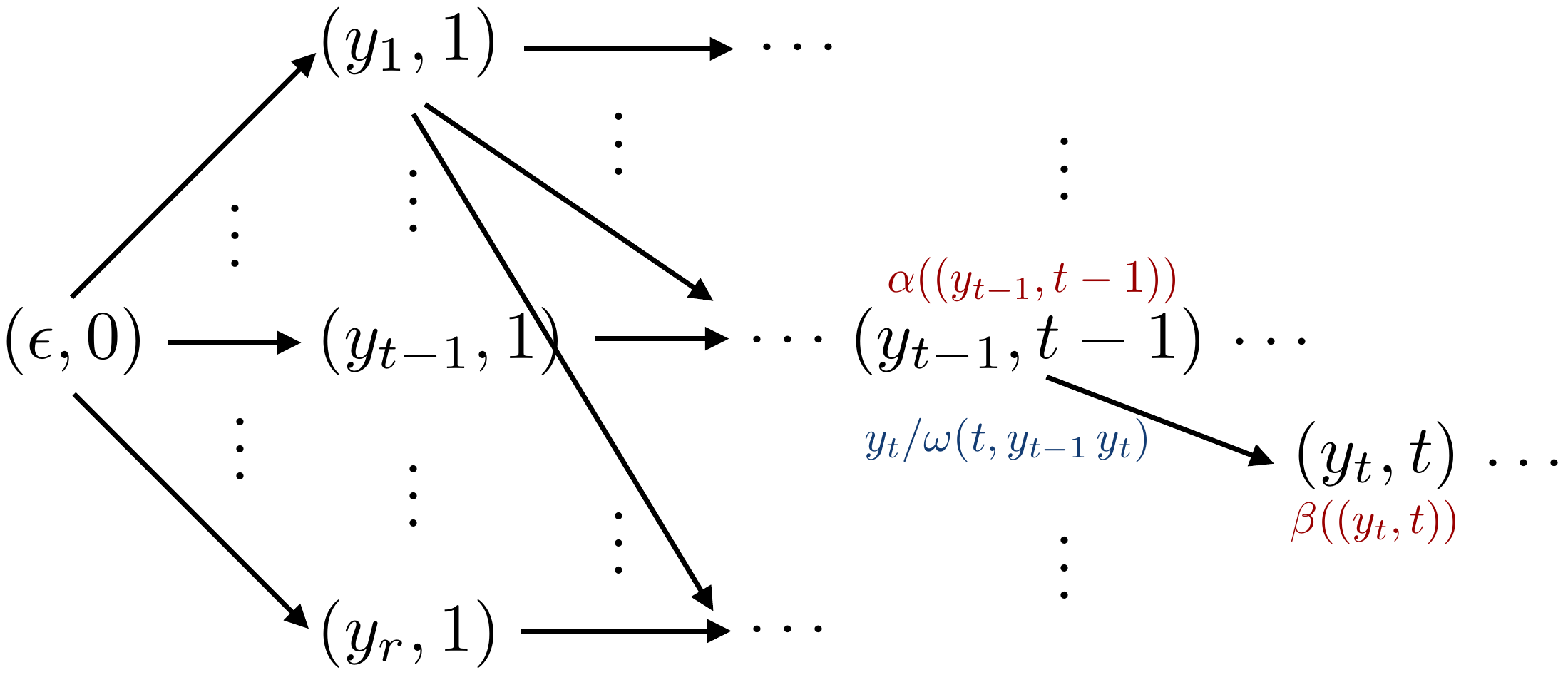}
\caption{Illustration of WFA $\scrA$ for $p = 2$.}
\label{fig:wfa}
\end{figure}

For any state $(y_{t - p + 1}^t, t) \in Q_\scrA$, let
$\alpha((y_{t - p + 1}^t, t))$ denote the sum of the weights of
all paths in $\scrA$ from $I_\scrA$ to $(y_{t - p + 1}^t, t)$
and $\beta((y_{t - p + 1}^t, t))$ the sum of the weights of all
paths from $(y_{t - p + 1}^t, t)$ to a final state. Then,
$\QQ'_\bw(\bz, s)$ is given by
\begin{equation*}
\QQ'_\bw(\bz, s) = \alpha \big((\bz_1^{p - 1}, s - 1) \big) \times
\omega(\bz, s)
\times \beta \big( (\bz_2^p, s) \big).
\end{equation*}
Note also that $Z'_\bw$ is simply the sum of the
weights of all paths in $\scrA$, that is $Z'_\bw = \beta((\ve, 0))$.

Since $\scrA$ is acyclic, $\alpha$ and $\beta$ can be computed for all
states in linear time in the size of $\scrA$ using a single-source
shortest-distance algorithm over the $(+, \times)$ semiring or the
so-called forward-backward algorithm. Thus, since $\scrA$ admits
$O(l |\Delta|^p )$ transitions, we can compute all of the quantities
$\QQ'_\bw(\bz, s)$, $s \in [l]$ and $z \in \Delta^p$ and $Z'_\bw$, in
time $O(l |\Delta|^p )$.

\subsubsection{Gradient computation with a Markovian loss}
\label{sec:GradMarkov}

We will say that a \emph{loss function $\loss$ is Markovian} if it
admits a decomposition similar to the features, that is
for all $y, y' \in \cY$,
\begin{equation*}
\loss(y, y') = \sum_{t = 1}^l \loss_t(y_{t - p + 1}^t, {y'}_{t -
  p + 1}^{t}).
\end{equation*}
In that case, we can absorb the losses in the transition
weights and define new transition weights $\omega'$ as follows:
\begin{align*}
\omega'(t, y_{t - p + 1}^{t - 1} \,
  b) = e^{L_t(y_{t - p + 1}^{t - 1} \,
  b, {(y_i)}_{t - p + 1}^{t - 1} \,
  b)}  \omega(y_{t - p + 1}^{t - 1} \,
  b, t).
\end{align*}
Using the resulting WFA $\scrA'$ and precisely the same techniques as
those described in the previous section, we can compute all
$\QQ_\bw(\bz, s)$ in time $O(l |\Delta|^p )$.  In particular, we can
compute efficiently these quantities in the case of the Hamming loss
which is a Markovian loss for $p = 1$.

\subsection{Efficient gradient computation for \VStructBoost}
\label{sec:Grad_vsb}

In this section, we briefly describe the gradient computation for
\VStructBoost, which follows along similar lines as the discussion for
\VCRF.

Fix $i \in [m]$ and let $F_i$ denote the contribution of point $x_i$
to the empirical loss in \VStructBoost.  Using the equality
$\loss(y_i, y_i) = 0$, $F_i$ can be rewritten as
\begin{align*}
F_i(\bw)
& = \frac{1}{m} \sum_{y \neq y_i} \loss(y, y_i)
e^{-\bw \cdot \d \bPsi(x_i, y_i, y) } 
 = \frac{1}{m} e^{-\bw \cdot \bPsi(x_i, y_i) } 
\sum_{y \in \Delta^l} \loss(y, y_i)
e^{\bw \cdot \bPsi(x_i, y) }.
\end{align*}
The gradient of $F_i$ can therefore be expressed as follows:
\begin{align}
\label{eq:vsb_grad}
\nabla F_i(\bw) 
&  = \frac{1}{m} e^{-\bw \cdot \bPsi(x_i, y_i) } 
\sum_{y \in \Delta^l} \loss(y, y_i)
e^{\bw \cdot \bPsi(x_i, y) } \bPsi(x_i, y) \\
& \quad - \frac{1}{m} e^{-\bw \cdot \bPsi(x_i, y_i) } 
\bPsi(x_i, y_i) \sum_{y \in \Delta^l} \loss(y, y_i)
e^{\bw \cdot \bPsi(x_i, y) }. \nonumber
\end{align}
Efficient computation of these terms is not straightforward, since
the sums run over exponentially many sequences $y$. However,
by leveraging the Markovian property of the features, we can reduce
the calculation to flow computations over a weighted 
directed graph, in a manner analogous to what we demonstrated for 
\VCRF. 

\begin{figure}[t] 
\centering
\includegraphics[scale=.31]{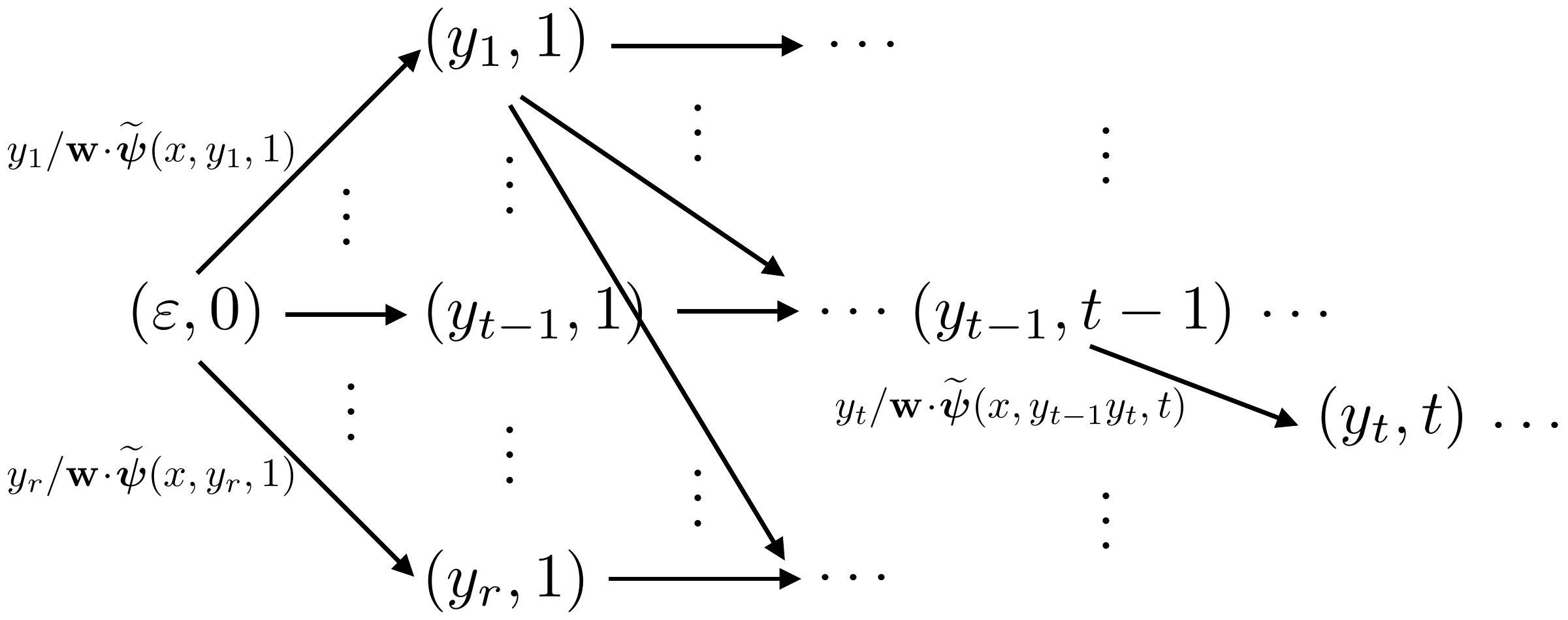}
\caption{Illustration of the WFA $\scrA'$ for $p = 2$.}
\label{fig:inference}
\end{figure}

\subsection{Inference}

In this section, we describe an efficient algorithm for inference when
using Markovian features. The algorithm consists of a standard
single-source shortest-path algorithm applied to a WFA $\scrA'$
differs from the WFA $\scrA$ only by the weight of each transition,
defined as follows:
\begin{align*}
E_{\scrA'}
= \Big\{ \Big( & (\bar y_{t - p + 1}^{t - 1}, t - 1), b, \bw \cdot \tl \bpsi(x, y_{t - p + 1}^{t - 1} b, t),
          (\bar
  y_{t - p + 2}^{t - 1} \, b, t) \Big) \colon 
y \in \Delta^l, b \in \Delta, t \in [l] \Big\}.
\end{align*}
Furthermore, here, the weight of a path is obtained by adding the
weights of its constituent transitions.  Figure~\ref{fig:inference}
shows $\scrA'$ in the special case of $p = 2$. By
construction, the weight of the unique accepting path in $\scrA'$ labeled with
$y \in \Delta^l$ is
$\sum_{t = 1}^l \bw \cdot \tl \bpsi(x, y_{t - p + 1}^{t - 1} b, t)
= \bw \cdot \bPsi(x, y)$.

Thus, the label of the single-source shortest path,
$\argmin_{y \in \Delta^l} \bw \cdot \bPsi(x, y)$, is the desired
predicted label. Since $\scrA'$ is acyclic, the running-time
complexity of the algorithm is linear in the size of $\scrA'$, that is
$O(l |\Delta|^l)$.

\newpage
\bibliographystyleapp{abbrvnat}
\begin{small} 
\bibliographyapp{vcrf}
\end{small}
\end{document}